\documentclass[arxiv]{tiltpaper}

\PassOptionsToPackage{numbers,square}{natbib}
\usepackage[utf8]{inputenc}
\usepackage[T1]{fontenc}
\usepackage{hyperref}
\usepackage{url}
\usepackage{booktabs}
\usepackage{amsfonts}
\usepackage{nicefrac}
\usepackage{microtype}
\usepackage{xcolor}
\usepackage{graphicx}
\usepackage[percent]{overpic}
\usepackage{mathtools}
\usepackage{amsmath,amssymb,amsthm,amsfonts,bbm}
\usepackage{mathbbol}
\usepackage{mathrsfs}

\usepackage{multicol}
\usepackage{rotating}
\usepackage{enumitem,comment, xifthen}
\usepackage[mathscr]{euscript}
\usepackage[most]{tcolorbox}

\usepackage{xspace}

\providecommand{\figpanel}[3]{%
  \begin{minipage}[t]{#1}
    \centering
    \begin{overpic}[width=\linewidth]{#3}
      \put(1,94){\scriptsize\bfseries #2}
    \end{overpic}
  \end{minipage}%
}
\providecommand{\figpanelbox}[3]{%
  \begin{minipage}[t]{#1}
    \centering
    \begin{overpic}[width=\linewidth]{#3}
      \put(-2,21){\begingroup\setlength{\fboxsep}{1pt}\colorbox{white}{\scriptsize\bfseries #2}\endgroup}
    \end{overpic}
  \end{minipage}%
}
\usepackage{algorithmic}
\usepackage{algorithm}

\usepackage{aliascnt}

\usepackage{wrapfig}
\hypersetup{ colorlinks, hypertexnames=false, linkcolor={red!72!black},
  citecolor={blue!72!black}, urlcolor={magenta!72!black} }

\usepackage[nameinlink]{cleveref}
\crefname{theorem}{theorem}{theorems}
\Crefname{theorem}{Theorem}{Theorems} \crefname{lemma}{lemma}{lemmas}
\Crefname{lemma}{Lemma}{Lemmas}
\crefname{proposition}{proposition}{propositions}
\Crefname{proposition}{Proposition}{Propositions}
\crefname{corollary}{corollary}{corollaries}
\Crefname{corollary}{Corollary}{Corollaries}
\crefname{definition}{definition}{definitions}
\Crefname{definition}{Definition}{Definitions}
\crefname{assumption}{assumption}{assumptions}
\Crefname{assumption}{Assumption}{Assumptions}
\crefname{remark}{remark}{remarks} \Crefname{remark}{Remark}{Remarks}

\usepackage{iiw_macros}
\usepackage{new_tilt_macros}

\makeatletter
\long\def\@makecaption#1#2{
  \vskip\abovecaptionskip
  \small #1: #2\par
  \vskip\belowcaptionskip
}
\renewcommand{\section}{%
  \@startsection{section}{1}{\z@}%
                {-1.85ex \@plus -0.45ex \@minus -0.2ex}%
                { 1.25ex \@plus  0.25ex \@minus  0.2ex}%
                {\large\bf\raggedright}%
}
\renewcommand{\subsection}{%
  \@startsection{subsection}{2}{\z@}%
                {-1.6ex \@plus -0.45ex \@minus -0.2ex}%
                { 0.65ex \@plus  0.2ex}%
                {\normalsize\bf\raggedright}%
}
\renewcommand{\paragraph}{%
  \@startsection{paragraph}{4}{\z@}%
                {1.1ex \@plus 0.4ex \@minus 0.2ex}%
                {-1em}%
                {\normalsize\bf}%
}
\makeatother

\providecommand{\mygraybox}[1]{%
{\begin{tcolorbox}[
    colback=gray!10,
    colframe=gray!70,
    arc=4pt,
    boxrule=0.7pt,
    left=6.5pt,
    right=6.5pt,
    top=3.5pt,
    bottom=3.5pt,
    before skip=5pt plus 1pt minus 1pt,
    after skip=5pt plus 1pt minus 1pt,
    enhanced jigsaw,
    parbox=false
]
#1
  \end{tcolorbox}
  }
}

\AtBeginDocument{
  \setlength{\parskip}{4.5pt plus 0.5pt minus 0.5pt}
  \setlength{\textfloatsep}{8pt plus 1pt minus 2pt}
  \setlength{\floatsep}{7pt plus 1pt minus 2pt}
  \setlength{\intextsep}{7pt plus 1pt minus 2pt}
  \setlength{\abovecaptionskip}{3.5pt plus 1pt minus 1pt}
  \setlength{\belowcaptionskip}{0pt}
  \setlength{\abovedisplayskip}{5.5pt plus 1.5pt minus 1.5pt}
  \setlength{\belowdisplayskip}{5.5pt plus 1.5pt minus 1.5pt}
  \setlength{\abovedisplayshortskip}{3.5pt plus 1pt minus 1pt}
  \setlength{\belowdisplayshortskip}{3.5pt plus 1pt minus 1pt}
}

\newcommand{\lampar}{\ensuremath{\lambda}}

\newcommand{\ffun}{f}  \newcommand{\bfun}{b}

\newcommand{\Vlam}{\ensuremath{v_\lampar}}
\newcommand{\Wlam}{\ensuremath{w_\lampar}}

\newcommand{\pdens}{\ensuremath{p}}
\newcommand{\qdens}{\ensuremath{q}}

\newcommand{\ERisk}{\ensuremath{\Exs\,\Risk}}
\newcommand{\ERiskHat}{\ensuremath{\Exs\,\RiskHat}}

\newcommand{\bstar}{\bfun^*}

\newcommand{\tilt}{\texttt{TILT}\xspace}
\newcommand{\kd}{\texttt{KD}\xspace}
\newcommand{\iw}{\texttt{IW}\xspace}
\newcommand{\kdtilt}{\texttt{KD-TILT}\xspace}
\newcommand{\kltilt}{\texttt{KL-TILT}\xspace}

\newcommand{\rl}{\texttt{RuLSIF}\xspace}

\newcommand{\Errlam}{\Err_\lambda^2}
\newcommand{\Errq}{\Err_\Target^2}

\newcommand{\Jpsilam}{\mathcal{J}_{\psi,\lambda}}
\newcommand{\Riskpsi}{\Risk_{\psi}} \newcommand{\bflam}{\bfun_{f,
    \lambda}}

\newcommand{\nlam}{\ensuremath{n_\lambda}}

\PaperBibliographyStyle{alpha}
\PaperBibliography{references}
\PaperBibliography{new_refs}

\PaperTitle{\tilt: Target-induced loss tilting under covariate shift}

\PaperAuthorsNeurIPS{%
  Kakei Yamamoto \\
  Lab for Information and Decision Systems \\
  Statistics and Data Science Center \\
  EECS, Massachusetts Institute of Technology \\
  \texttt{kakei@mit.edu}
  \And
  Martin J.~Wainwright \\
  Lab for Information and Decision Systems \\
  Statistics and Data Science Center \\
  Mathematics and EECS, Massachusetts Institute of Technology \\
  \texttt{mjwain@mit.edu}
}

\PaperAuthorsArxiv{%
  \begin{tabular}{ccc}
    Kakei Yamamoto$^\dagger$ && Martin J.~Wainwright$^\star$\\
    \texttt{kakei@mit.edu} && \texttt{mjwain@mit.edu}
  \end{tabular}

  \vspace*{0.2in}
  \begin{tabular}{c}
    Lab for Information and Decision Systems \\
    Statistics and Data Science Center \\
    EECS$^{\dagger,\star}$ and Mathematics$^\star$ \\
    Massachusetts Institute of Technology
  \end{tabular}
}

\PaperAbstract{We introduce and analyze Target-Induced Loss Tilting (\tilt) for
unsupervised domain adaptation under covariate shift.  It is based on
a novel objective function that decomposes the source predictor as
$f+b$, fits $f+b$ on labeled source data while simultaneously
penalizing the auxiliary component $b$ on unlabeled target inputs. The
resulting fit $f$ is deployed as the final target predictor. At the
population level, we show that this target-side penalty implicitly
induces relative importance weighting at the population level, but in
terms of an estimand $b^*_f$ that is \emph{self-localized} to the
current error, and remains \emph{uniformly bounded} for any
source-target pair (even those with disjoint supports).  We prove a
general finite-sample oracle inequality on the excess risk, and use it
to give an end-to-end guarantee for training with sparse ReLU
networks. Experiments on controlled regression problems and shifted
CIFAR-100 distillation show that \tilt improves target-domain
performance over source-only training, exact importance weighting, and
relative density-ratio baselines, with a stable dependence on the
regularization parameter.
}

\begin{document}
\MakePaperTitle

\section{Introduction}
\label{SecIntro}

Many prediction systems are trained on labeled data from a source
domain and then deployed on inputs drawn from a different target
domain. This mismatch arises in applications such as medical diagnosis
across patient populations, animal classification across acquisition
environments, and poverty prediction across geographic settings (e.g.,
see the survey paper~\citep{koh2021wilds} and references therein). We
study this problem under covariate shift: the source and target
covariate distributions differ, but the conditional law of the
response given the covariate is
shared~\citep{shimodaira2000improving,sugiyama2012density}. In this
setting, ordinary source empirical risk minimization need not target
the relevant prediction criterion, since it weights errors according
to the source marginal rather than the target marginal.

The standard correction is importance weighting. If $\pdens$ and
$\qdens$ denote the source and target covariate densities,
respectively, then the target risk can be expressed as a source
expectation weighted by the density ratio $\qdens(x)/\pdens(x)$. This
identity is exact, but its direct use has two drawbacks. First, the
ratio can be highly variable or unbounded, which increases statistical
variance and can destabilize optimization. Second, the ratio is
typically unknown and must be estimated from unlabeled source and
target samples. Direct density-ratio methods such as KMM, KLIEP, and
LSIF avoid separately estimating $\pdens$ and $\qdens$, but still
require an explicit weight-estimation stage
\citep{huang2007correcting,gretton2008covariate,sugiyama2008direct,
  kanamori2009least}. Relative density-ratio methods such as RuLSIF
replace the ordinary ratio by a regularized ratio to improve stability
\citep{yamada2011relative}, while one-step approaches jointly optimize
prediction and weighting components \citep{zhang2020onestep}. These
methods differ algorithmically, but they remain explicitly
weight-based.  See~\Cref{SecRelated} for a more detailed discussion
and this and other related work. \\

In this paper, we propose a novel method known as
\textbf{Target-Induced Loss Tilting} (\tilt), a one-step method for
covariate-shift adaptation that removes density-ratio estimation from
the algorithm. The \tilt method decomposes the source predictor as
$f+b$. The sum $f+b$ is fit on labeled source data, while the
auxiliary component $b$ is penalized on unlabeled target covariates;
the deployed predictor is $f$. The role of $b$ is to absorb structure
useful for fitting the source distribution, subject to the constraint
that this auxiliary structure is suppressed on the target covariate
distribution. Thus the target sample enters only through a penalty on
the auxiliary component, not through estimated density ratios or
target labels.

A key technical property of the \tilt method is that this target-side
penalty induces a relative weighting criterion after profiling out the
auxiliary component. For the least-squares population objective,
optimizing over $b$ yields an exact $\lambda$-relative weighted target
excess risk.  For a given predictor $f$, the corresponding optimal
offset is given by $\bstar_f(x) = - \frac{\pdens(x)}{\pdens(x)+\lambda
  \qdens(x)} \{f(x)-f^\star(x)\}$, where $\fstar(x) = \Exs[Y\mid X=x]$
is the optimal predictor. This expression involves the \emph{product}
of the two terms: the bounded factor
$\pdens(x)/(\pdens(x)+\lambda\qdens(x))$, rather than the possibly
unbounded density ratio $\qdens(x)/\pdens(x)$, and the current
prediction error $f - \fstar$. Consequently, the \tilt method has a
\emph{self-localizing property}: the auxiliary task becomes easier as
the main predictor $f$ improves, and the resulting procedure
implements $\fstar$-targeted covariate-shift correction without
explicitly estimating $\qdens/\pdens$.

\myparagraph{Contributions} Let us summarize our contributions. First,
we introduce \tilt, a one-step procedure for covariate-shift
adaptation that fits an additive predictor $f+b$ on labeled source
data, penalizes the auxiliary component $b$ on unlabeled target
covariates, and deploys only $f$.  Unlike importance-weighting
methods, \tilt never estimates, clips, or outputs a density ratio.

Second, we provide a rigorous theoretical under-pinning for the
method.  Our first theoretical result (\Cref{ThmRiskDecomp}) shows
that this simple target-side penalty has attractive population-level
properties. Profiling out the auxiliary component $\bfun$ yields a
$\lambda$-relative weighted target excess risk. Moreover, the optimal
offset $\bstar_f$ is uniformly bounded for arbitrary source--target
pairs and is localized to the current prediction error. Thus, \tilt
implements covariate-shift correction without requiring density-ratio
estimation, and in a way that is self-localized to the error $f -
\fstar$.

In addition, we give two finite-sample guarantees for the data-based
procedure used in practice: \Cref{ThmExcessRisk} gives a general
oracle inequality that separates approximation error from the source
and target estimation terms, and~\Cref{CorExcessRisk} specializes this
general result to sparse ReLU networks to obtain an end-to-end
nonparametric rate. Our bounds make explicit how the regularization
parameter $\lambda$ trades target relevance against estimation
variance, and how various smoothness conditions affect the rates.

On the empirical side, we evaluate \tilt through a suite of numerical
studies. In controlled regression experiments with known source and
target densities, we compare against source ERM, as well as the
\emph{oracle forms} of importance weighting and relative-density-ratio
that are given knowledge of the pair $(p,q)$.  We construct an
experiment demonstrating that \tilt achieves \emph{minimax-optimal}
rates in terms of covariate shift severity.  Finally, in shifted
CIFAR-100 distillation, we extend the same idea to auxiliary logits
and show that target-side tilting improves target performance over
source-only training and KD methods under image-level covariate shift.

\myparagraph{Organization} \Cref{SecPreliminaries} sets up the
covariate-shift problem and introduces the \tilt
objective. \Cref{SecMethods} proves the population reweighting
identity, and \Cref{SecExtensions} records the classification variant
used in the distillation experiments.  \Cref{SecTheory} gives the main
finite-sample and ReLU-network guarantees. \Cref{SecNumerical}
evaluates \tilt on synthetic regression problems and shifted
CIFAR-100.

\subsection{Related work}
\label{SecRelated}

Here we elaborate on three threads of work related to the \tilt
approach; due to space constraints, we limit ourselves to those papers
most directly relevant.

\myparagraph{Importance weighting} The standard correction for
covariate shift is importance-weighted (IW) risk
minimization~\citep{shimodaira2000improving,sugiyama2012density},
which rewrites target risk as a source expectation weighted by
$w(x)=\qdens(x)/\pdens(x)$.  Generalization bounds for IW methods show
that difficulty depends on the magnitude and variability of the weight
function~\citep{cortes2010learning,MaPatWai23}.  Since the density
ratio is typically unknown, a large literature estimates it directly
from unlabeled source and target samples, including
KMM~\citep{huang2007correcting,gretton2008covariate},
KLIEP~\citep{sugiyama2008direct}, LSIF~\citep{kanamori2009least}, and
related convex risk formulations~\citep{nguyen2010estimating}.  Zhang
et al.~\citep{zhang2020onestep} replace the two-stage
importance-weighting pipeline with a joint optimization over
predictors and weights, but still explicitly parameterize and learn
the density ratio.  Relative density-ratio methods such as RuLSIF
replace the ordinary ratio by a regularized variant so as to improve
stability at the expense of bias~\citep{yamada2011relative}. The \tilt
method involves a related form of regularization, but \emph{does not}
estimate or output density ratios: the relative weights appear only
after the auxiliary component is optimized out, and the optimal offset
$\bstar_f$ couples it with the estimate $f$
(cf.~\Cref{ThmRiskDecomp}).

\myparagraph{Domain adaptation and distributional robustness} A
broader line of domain-adaptation work seeks representations or
predictors whose behavior is stable across domains. Classical theory
bounds target error using source error and a discrepancy between
domains \citep{ben2010theory}. Kernel and moment-matching approaches
control distributional discrepancy through quantities such as
MMD~\citep{gretton2012kernel}, while deep adaptation methods align
features or domains using adversarial losses, correlation matching, or
residual transfer~(e.g.,\cite{ganin2016domain,long2015learning,
  sun2016deep,long2016unsupervised}). Invariant and robust prediction
methods, including invariant risk minimization, risk extrapolation,
and anchor regression, instead penalize predictors whose performance
or residuals vary across observed environments or anchor
variables~\cite{arjovsky2019invariant,krueger2021out,rothenhausler2021anchor}.
All of these methods are based on a different observation model than
ours: they rely on observed environments, anchors, or multiple domains
that reveal how the distribution changes. In contrast, the \tilt
procedure requires only labeled source data and unlabeled target
covariates, with the target covariates serving as the adaptation
signal.

\myparagraph{Additive decompositions and residual transfer} Additive
shared-specific decompositions have been used in domain adaptation,
multi-task learning, and representation learning. Feature augmentation
methods introduce shared and domain-specific copies of
features~\citep{daume2007frustratingly}; multi-task models decompose
parameters into shared and task-specific
components~\citep{evgeniou2004regularized,jalali2010dirty}; and domain
separation networks split representations into shared and private
factors~\citep{bousmalis2016domain}. Residual transfer methods adapt a
source predictor by learning an additional residual component for the
target domain~\citep{kuzborskij2013stability,long2016unsupervised}, or
by using residuals to reduce bias in student--teacher
estimation~\citep{yamamoto2026residual}.  The \tilt method also
exploits an additive decomposition, but with the opposite deployment
logic: the auxiliary component is allowed to fit source-specific
residual structure while being penalized on target inputs, and is
discarded when constructing the final predictor. By design, this
choice yields an implicit covariate-shift correction rather than a
shared--private architectural bias.

\section{Problem and method formulation}
\label{SecPreliminaries}

We consider the unsupervised domain adaptation problem, with $\Xspace
\subset \mathbb{R}^d$ being the input space and $\Yspace \subset
\mathbb{R}$ the output space. We are given a source dataset $\{(x_i,
y_i)\}_{i=1}^\numsource$ consisting of $\numsource$ labeled examples
drawn i.i.d. from a source distribution $\Source$ over $\Xspace
\times \Yspace$. Additionally, we have access to a target dataset
$\{\tilde{x}_j\}_{j=1}^\numtarget$ consisting of $\numtarget$
unlabeled examples drawn i.i.d. from the marginal distribution
$\TargetX$ of a target distribution $\Target$.  In the covariate shift
setting, the marginal distributions of the inputs differ ($\SourceX
\neq \TargetX$), but the conditional distribution of the labels
remains invariant: $\Source(Y|X) = \Target(Y|X)$.

\subsection{\tilt: Target-induced loss tilting}
\label{SecMethods}

We begin by describing the \tilt procedure for real-valued
outputs $y \in \real$, and using the least-squares loss.
In this setting, the optimal predictor is given by the
regression function $\fstar(x) = \Exs[Y \mid X = x]$, and the
quality of an estimate $f$ is measured by the $\Qprob$-target
excess risk
\begin{align}
\label{EqnTargetRisk}  
\Errq(f) = \Exs_\Qprob \big(Y - f(\Xtil) \big)^2 - \Exs_\Qprob \big(Y
- \fstar(\Xtil) \big)^2 \; \equiv \Exs_\Qprob \big(f(\Xtil) -
\fstar(\Xtil) \big)^2
\end{align}

We introduce an additive decomposition of the form $f(x) + b(x)$,
where $f \in \Fclass$ is the predictor and $b \in \Bclass$ is an
auxiliary function. We propose to learn $f$ and $b$ simultaneously by
minimizing the following joint empirical objective:
\mygraybox{
  \begin{align}
 \label{EqnTILT}
    \mbox{\textbf{\tilt objective:}}\qquad \LossHat(f,b) \coloneqq
    \frac{1}{\numsource} \Sum{i}{\numsource} \left( f(x_i) - y_i +
    b(x_i) \right)^2 + \frac{\lambda}{\numtarget} \Sum{j}{\numtarget}
    b^2(\xtil_j).
\end{align}
}
\noindent
We jointly minimize $\LossHat$ over $(f,b)$ and use the $f$-component
of a minimizer as the final \tilt predictor.  Here the regularization
parameter $\lambda > 0$ controls the strength of bias suppression on
the target domain.  Note that the objective~\eqref{EqnTILT} is simple,
and remains jointly convex in the function values $f(x_i)$,
$\bfun(x_i)$, and $\bfun(\xtil_j)$.

\myparagraph{Theoretical under-pinning: excess risk and optimal
  offset} Our first result shows that the \tilt
objective~\eqref{EqnTILT} has a rigorous justification at the
population level of infinite sample size.  Define the population level
\tilt objective $\LossBar(f, \bfun) \defn \Exs[\LossHat(f, \bfun)]$,
where the expectation is taken over both source and target data, and
note that $\LossBar$ is minimized by the pair $(\fstar, 0)$, where
$\fstar(x) = \Exs[Y \mid X = x]$ is the optimal predictor.  The
\emph{auxiliary excess \tilt-risk} is given by
\begin{align}
\label{EqnDefnExcessRisk}
\Rbar(f, b) \defn \LossBar(f, \bfun) - \LossBar(\fstar, 0),
\end{align}
which compares the loss of a pair $(f, \bfun)$ relative to the oracle
optimum $(\fstar, 0)$. The following result shows minimizing the
auxiliary excess risk over the choice of $\bfun$ yields the
$\lambda$-weighted error function
\begin{align}
\label{EqnDefnErrlam}  
  \Errlam(f) & \defn \Exs_\Qprob \Big[ \frac{\pdens(x)}{\pdens(x) +
      \lampar \qdens(x)} \big(f(X) - \fstar(X) \big)^2 \Big]
\end{align}
and that moreover, the optimal offset function $\bstar_f$ has an
especially attractive and self-localized form.

\mygraybox{
\begin{proposition}[Excess risk and optimal offset]
\label{ThmRiskDecomp}
For any scalar $\lambda > 0$, function $f$ and density pair $(\pdens,
\qdens)$, we have the equivalence
\begin{subequations}
\begin{align}
\label{EqnOptExcessRisk}    
\inf_{\bfun} \tfrac{1}{\lampar} \Rbar(f, \bfun) & = \Errlam(f)
\end{align}
with the infimum achieved by the {\bf{optimal offset function}}
\begin{align}
\label{EqnOptAuxiliary}  
\bstar_f(x) \coloneqq - \Vlam(x) (f(x) - \fstar(x)) \qquad \mbox{where
  $\Vlam(x) \defn \frac{\pdens(x)}{(\pdens(x) + \lambda
    \qdens(x))}$.}
\end{align}
\end{subequations}
\end{proposition}
}
\noindent See~\Cref{SecThmRiskDecomp} for the proof.

\Cref{ThmRiskDecomp} highlights several key advantages of \tilt over
classical density-ratio-based methods.   \\
\noindent{\bf{Correctness of profiled excess risk:}} Minimizing the
auxiliary excess risk $\Risk(\ffun, \bfun)$ over the offset term
$\bfun$ defines a profiled excess risk.
Equation~\eqref{EqnOptExcessRisk} shows that this profiled risk is
equivalent to $\Errlam(f)$, showing that \tilt directly targets the
relative reweighted risk without explicit density-ratio estimation.
Moreover, the risk $\Errlam(f)$ from equation~\eqref{EqnDefnErrlam}
converges to the $\Qprob$-target excess risk~\eqref{EqnTargetRisk} as
$\lambda$ decreases. \\
\begin{wrapfigure}{r}{0.56\textwidth}
  \vspace{-8pt}
  \centering
  \begingroup
  \setlength{\tabcolsep}{2pt}
  \begin{tabular}{@{}cc@{}}
    \includegraphics[keepaspectratio,height=0.24\textwidth]{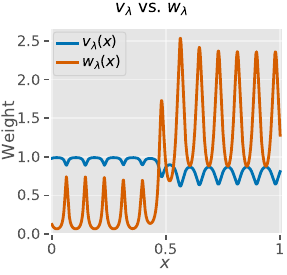} &
    \includegraphics[keepaspectratio,height=0.24\textwidth]{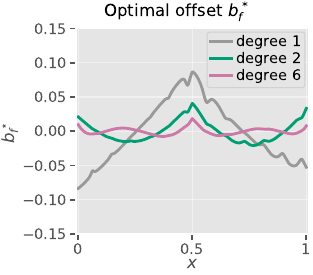}
  \end{tabular}
  \endgroup
  \caption{Left panel: The $\lambda$-smoothed density ratio $\Wlam(x)
    = \frac{q(x)}{p(x) + \lambda q(x)}$ is very spiky, compared to the
    relatively flat and well-behaved \tilt-weight $\Vlam(x) =
    \frac{p(x)}{p(x) + \lambda q(x)}$.  Right panel: Plots of the
    optimal offset function $\bstar_f$ for three different choices of
    $f$: linear, quadratic and degree six polynomial approximations to
    $\fstar$.  The optimal offset is much smaller and smoother than
    $\Wlam$; note the different scale of the $y$-axis.}
  \vspace{-8pt}
  \label{FigBstar}
\end{wrapfigure}
  \noindent{\bf{Uniform boundedness:}} The optimal offset $\bstar_f$
  does \emph{not} involve the density ratio, which can be very large
  or unbounded.  Instead, it involves the weight $\Vlam$ from
  equation~\eqref{EqnOptAuxiliary}, which takes values in $[0,1]$ for
  any choice of $(\pdens, \qdens, \lambda)$.  Thus, unlike explicit
  density-ratio estimation, \tilt never requires learning an unbounded
  or otherwise ill-behaved quantity.\\
  \noindent{\bf{Self-localizing property:}} Equally important is that the
  optimal offset $\bstar_f$ involves the product of $v_\lambda(x)$
  with the \emph{residual prediction error} $\Delta(x) = f(x) -
  \fstar(x)$.  Consequently, the offset $\bstar_f$ shrinks as the
  prediction error decreases, so that the auxiliary task becomes
  easier as $f$ improves.  This property corresponds to a form of
  self-localization: the optimal offset $\bstar_f$ is correctly
  localized around the current prediction error.

\subsection{Classification extension of \tilt}
\label{SecExtensions}

Thus far, we have described the \tilt method for the least-squares
setting.  Here we describe its extension to classification, which
underlies the CIFAR-100 distillation experiments described
in~\Cref{SecNumerical}. At a high level, the \kdtilt is simply the
knowledge-distillation analogue of \tilt: it keeps the same
source-side prediction objective, while adding a target-side
regularization that suppresses the auxiliary component on unlabeled
target inputs.

For a $K$-class logistic regression problem, we write the logits as
$f(x)+b(x)$, where $f,b:\Xspace\to\mathbb{R}^K$, and let
$\tau:\Xspace\to\mathbb{R}^K$ be teacher logits. Given a temperature
$T>0$ and a mixing weight $\beta\in[0,1]$, we define
\begin{align}
    \inf_{f,b} & \Biggr \{ \frac{1}{\numsource} \Sum{i}{\numsource}
    \Bigl[ (1-\beta)\, \ell_{\mathrm{CE}} \Bigl( y_i,\,
      \softmax\bigl(f(x_i)+b(x_i)\bigr) \Bigr) + \beta T^2
      \underbrace{\mathrm{KL} \Bigl( \softmax_T\bigl(\tau(x_i)\bigr)
        \,\Big\|\, \softmax_T\bigl((f(x_i)+b(x_i))\bigr)
        \Bigr)}_{\text{\kd loss}} \Bigr] \notag\\ & +
    \frac{\lambda}{\numtarget} \Sum{j}{\numtarget} T^2
    \underbrace{\mathrm{KL} \Bigl(
      \softmax_T\bigl((f(\xtil_j)+b(\xtil_j))\bigr) \,\Big\|\,
      \softmax_T\bigl(f(\xtil_j)\bigr) \Bigr)}_{\text{\tilt penalty}}
    \Biggr \},
    \label{EqForwardStabilizedSurrogate}
\end{align}
where $\softmax$ is the softmax function, and $\softmax_T(\cdot)\defn
\softmax(\cdot/T)$ is the softmax at temperature $T$.  The first two
terms are the standard supervised and soft distillation losses on the
source domain, whereas the last term is the \tilt component.
In~\Cref{SecGeneralization}, we derive the \kltilt procedure as a
direct KL-divergence analogue of \tilt, and show how the \kdtilt
objective is obtained by mixing it with the ordinary supervised
cross-entropy loss.

\section{Learning-theoretic guarantees}
\label{SecTheory}

Having laid out the population (infinite data) behavior of the \tilt
procedure, we now to turn its finite-sample analysis.  We give two
main results: ~\Cref{ThmExcessRisk} is an oracle inequality on the
excess risk that applies to a general pair $(\Fclass, \Bclass)$ of
function classes for the main and auxiliary fits, involving their
metric entropies.  \Cref{CorExcessRisk} provides an explicit rate by
specializing this general result to the case of sparse deep neural
nets.

The analysis in this section applies to the \tilt procedure for the
least-squares loss with real-valued response $Y \in \real$, and
optimal predictor $\fstar(x) \defn \Exs[Y \mid X = x]$.  We say that a
function class $\Fclass$ is $\bou$-uniformly bounded if $\|f\|_\infty
= \sup_{x \in \Xspace} |f(x)| \leq \bou$ for all $f \in \Fclass$.
Throughout this section, we impose the following {\bf{\emph{standard
    regularity}}} (SR) conditions:
  \begin{itemize}[
  noitemsep,
  topsep=0pt,
  leftmargin=*,
  align=left
    ]
  \item the function classes $\Fclass$ and $\Bclass$ are each $\bou$-uniformly bounded.
  \item the response noise $Y - \fstar(x)$ is sub-Gaussian with variance proxy $\sigma^2$.
  \end{itemize}
\noindent Recalling the \tilt objective function $\LossHat$ from
equation~\eqref{EqnTILT}, we let $(\fhat, \bhat) \in \arg
\min_{(\ffun, \bfun) \in \Fclass \times \Bclass} \LossHat(f, b)$
denote a joint minimizer, and the final \tilt estimator corresponds to
the function $\fhat$.

\myparagraph{Main excess risk bound} Our first result gives an oracle
inequality for the expected reweighted excess risk. The approximation
term measures how well $\Fclass$ and $\Bclass$ realize the population
decomposition, while the estimation error is quantified in terms of
metric entropy.  More precisely, for any $\delta > 0$, let $\log
\Num_\numsource \coloneqq \log \Num_\numsource(\delta; \Fclass +
\Bclass; L_\infty)$ and $\log \Num_\numtarget \coloneqq \log
\Num_\numtarget(\delta; \Bclass; L_\infty)$ denote the metric entropy
at scale $\delta$.  The following theorem gives a bound in terms of
the error metric $\Errlam(f)$ from equation~\eqref{EqnOptExcessRisk}.

\mygraybox{
\begin{theorem}[Non-asymptotic bound on $\Errlam(\fhat)$]
\label{ThmExcessRisk}
Under the SR conditions, there exists a universal constant $c > 0$
such that for any $s \in (0,1)$ and any $\delta \in (0,1)$, the \tilt
estimate $\fhat$ satisfies
\begin{align}
\label{EqnOracle}
 \Exs \big[ \Errlam(\fhat) \big] & \leq (1 + s) \underbrace{\inf_{f
     \in \Fclass, \bfun \in \Bclass} \tfrac{1}{\lambda} \Rbar(\ffun,
   \bfun)}_{\mbox{Approximation error}} + \frac{1}{s} \underbrace{c
   (B^2 + \sigma^2) \left( \frac{\log
     \Num_\numsource}{\lambda\numsource} + \frac{\log
     \Num_\numtarget}{\numtarget} \right)}_{\mbox{Estimation
     error}} + \MyRem(\delta),
\end{align}
where $\MyRem(\delta) = \frac{c'}{\lambda} (B + \sigma)\delta =
\mathcal{O}(\delta)$ is a remainder term.
\end{theorem}
}
\noindent
See~\Cref{AppThmExcessRisk} for the proof. A few comments on this
result are in order:
  \begin{itemize}[
  noitemsep,
  topsep=0pt,
  leftmargin=*,
  align=left
    ]
\item The first term in the bound~\eqref{EqnOracle} represents
  approximation error, as measured by the best performing pair
  $(\ffun, \bfun)$ for the auxiliary excess risk $\Risk$ from
  equation~\eqref{EqnDefnExcessRisk}.  From~\Cref{ThmRiskDecomp}, if
  the auxiliary class $\Bclass$ is rich enough to include $\bstar_f$
  (cf. equation~\eqref{EqnOptAuxiliary}) for each $f \in \Fclass$,
  then this approximation error takes the form $\inf_{f \in \Fclass}
  \Errlam(f)$, which is directly comparable to the left-hand side.
\item The parameter $\lambda$ controls the bias--variance tradeoff:
  smaller $\lambda$ moves the criterion toward the
  $L^2(\Qprob)$-target risk $\mathcal{E}_0(f) \equiv \Exs_{\Qprob}
  \big[(f(X) - \fstar(X))^2\big]$, while the bound pays for this
  through the source complexity term $\log \Num_\numsource$ scaled by
  $1/\lambda$.

  \end{itemize}

\myparagraph{ReLU network excess risk bounds} \Cref{ThmExcessRisk} is
a general result that applies to any $\bou$-bounded pair of function
classes $(\Fclass, \Bclass)$.  We can obtain explicit rates of
convergence by specializing it to particular classes, and here we
describe such a result for sparse clipped ReLU networks.  Our analysis
converts the general oracle inequality~\eqref{EqnOracle} into an
explicit nonparametric rate for this network class.

Let $\NNs(L,\mathbf{p},s, \bou)$ denote the class of neural networks
with $L$ hidden layers, width vector $\mathbf{p}$, at most $s$ nonzero
parameters, and final output bounded by $\bou$; see~\Cref{DefNN}
in~\Cref{AppCorExcessRisk} for details.  For pairs $(\mathbf{p}_f,
s_f)$ and $(\mathbf{p}_g, s_g)$ to be chosen, our theory applies to
the pair
\begin{subequations}
\begin{align}
  \label{EqnNeuralNets}
\mbox{\underline{Main class:}} \quad \Fclass = \NNs(L_f, \mathbf{p}_f,
s_f, \bou), \qquad \mbox{\underline{Auxiliary class}} \quad \Bclass =
\NNs(L_g, \mathbf{p}_g, s_g, \bou).
\end{align}
As for the data-generating model, we impose H\"{o}lder smoothness
conditions on the true regression function $\fstar: \real^d
\rightarrow \real$, and the bounded weight $\Vlam(x) =
\frac{p(x)}{p(x) + \lambda q(x)}$.  In particular, for smoothness
parameters $\beta, \gamma > 0$ and a radius $K$, we assume that
\begin{align}
\label{EqnDGM}  
  \fstar \in C_d^\beta([0,1]^d,K), \quad \mbox{and} \quad \Vlam=
  \pdens/(\pdens+\lambda\qdens) \in C_d^\gamma([0,1]^d,K),
\end{align}
and moreover that $\fstar$ is $K$-bounded with $K \le \bou$.
\end{subequations}
See~\Cref{DefHolder} in~\Cref{AppCorExcessRisk} for details of the
H\"older ball $C_d^\alpha([0,1]^d,K)$.

\mygraybox{
  \begin{corollary}[Excess risk bound for neural net classes]
\label{CorExcessRisk}
Under conditions~\eqref{EqnDGM} on the pair $(\fstar, \Vlam)$,
consider the \tilt procedure implemented with the neural net
classes~\eqref{EqnNeuralNets} with parameters
\begin{subequations}
\begin{align}
    s_f & \asymp \big(\kaplam \nlam\big)^{\frac{d}{2\beta+d}},
    \qquad s_g\asymp s_f +\big(\kaplam
    \nlam\big)^{\frac{d}{2\gamma+d}}, \qquad L_f \asymp L_g \asymp
    \log(\kaplam \nlam).
\end{align}
where $\nlam \coloneqq \min \{\lambda \numsource, \numtarget \}$,
and $\kaplam \coloneqq (1+\lambda)/\lambda$.  Then the \tilt estimate
$\fhat$ satisfies the following risk bound, with the displayed
$\lambda$-dependence entering through $\nlam$ and $\kaplam$:
\begin{align}
\label{EqnNnetRisk}  
    \Exs\Errlam(\fhat) & \leq c_0 (K^2 + \sigma^2) \Big \{
    \kaplam^{\frac{d}{2\beta+d}} \nlam^{-\frac{2\beta}{2\beta+d}} +
    \kaplam^{\frac{d}{2\gamma+d}} \nlam^{-\frac{2\gamma}{2\gamma+d}}
    \Big \} \; \big(\log(\kaplam \nlam)\big)^3
\end{align}
\end{subequations}
Here $c_0$ may depend on $d,\beta,\gamma,\bou$, but not on
$\lambda,\numsource$ or $\numtarget$.
\end{corollary}
}
The effective sample size is $\nlam$ because the auxiliary class
is controlled by both the source and target covering terms.  The
source term is the usual approximation-estimation rate for the
$\beta$-smooth predictor $\fstar$, while the target term is the
corresponding rate for the $\gamma$-smooth function $\Vlam$
controlling the auxiliary correction. A key property is that the
result \emph{does not} require the raw\footnote{In addition, the
$\lambda$-regularized version $q/(p + \lambda q)$ is bounded, but can
take values as large as $1/\lambda$ , in contrast to the uniform bound
$\|\Vlam\|_\infty \leq 1$.} density ratio $\qdens/\pdens$ to be
bounded or smooth.

\section{Numerical results}
\label{SecNumerical}

\subsection{Synthetic regression experiments}

The synthetic experiments are designed to isolate the mechanism of
\tilt in settings where the source and target distributions are known.
For a new loss-tilting mechanism, this control is essential: it
exposes the behavior of the objective apart from density-ratio
estimation error.  We use a Beta covariate shift in one dimension and
a Gaussian covariate shift in higher dimension, followed by a
point-mass nonparametric rate experiment.  We compare source ERM, exact
importance weighting (\iw), exact relative least-squares importance
fitting (\rl), kernel-estimated \rl when applicable, and \tilt.  The
exact \iw and \rl baselines use the known source and target
distributions; no density-ratio estimation or clipping is used for
these exact baselines.

\myparagraph{Linear misspecification} In the linear experiment, the
target predictor $f$ is a degree-three linear model in shifted
orthonormal Legendre features, while the auxiliary component $b$ is a
Gaussian-kernel ridge expansion on uniform centers.  The target
distribution is fixed at a Beta$(2,5)$ endpoint and the source
distribution is swept from the matched case toward a Beta$(5,2)$
endpoint over 21 corruption levels.  For each level we draw $320$
source samples, $320$ target samples, and $12{,}000$ test points, and
report means over $100$ trials.  All least-squares subproblems are
solved by ridge regression, and $\lambda$ is swept on a logarithmic
grid from $10^{-6}$ to $10^4$.  The response contains a degree-three
base component plus a localized high-frequency residual, chosen
deliberately so that source ERM with the misspecified linear class
extrapolates poorly into the target region.

From \Cref{FigSyntheticTiltLinearAB}\textbf{A}, the exact \iw
degrades, and exact \rl improves over \iw after corruption level
$0.5$, but \tilt remains the best method shown throughout the sweep.
Panel \textbf{B} shows the same \tilt estimator as a function of
$\lambda$ under a larger shift: the best values move to an
intermediate regime, visualizing the bias--variance tradeoff induced
by the \tilt objective.  As $\lambda \to \infty$, the target penalty
suppresses the auxiliary component $b$, and the estimator approaches
the source-ERM reference. For intermediate values of $\lambda$, the
target-side correction improves performance without exhibiting the
instability suggested by the $1/\lambda$ source-complexity term in the
bound.

\paragraph{High-dimensional neural misspecification}
We next test a high-dimensional Gaussian covariate-shift problem in
which the deployed target class is deliberately misspecified.  The
target distribution is fixed at $N(0,I_{4096})$, while the source
distribution interpolates from the matched case to
$N(\mu,I_{4096})$, where $\|\mu\|_2=3$ and the shift lies in a
$512$-dimensional coordinate subspace.  The regression function
contains many target-relevant ridge components together with a
source-local residual, so a weak predictor cannot represent the true
conditional mean.  Across all methods the deployed predictor $f$ is
the same one-hidden-layer ReLU network of width $16$, with effective
weight decay fixed at $10^{-5}$.  \tilt uses only an additional
auxiliary network $b$ with three hidden layers of width $256$.

For each shift level we use $1024$ labeled source samples, $1024$
unlabeled target samples for the \tilt penalty and density-ratio
estimation, $128$ labeled target validation samples for model
selection, and an independent target test set of size $8192$, averaged
over $20$ trials.  The validation labels are used only to select
hyperparameters.  Source ERM and \iw select the early-stopping epoch
from $\{1100,2200,4400\}$.  Exact \rl and kernel-estimated \rl select
$\lambda \in \{0.1,1,10,100\}$ and the same early-stopping grid.
\tilt selects $\lambda \in \{0.1,1,10,100\}$, the auxiliary effective
weight decay for $b$ from $\{10^{-7},10^{-6},10^{-5}\}$, and the same
early-stopping grid.  The architecture and effective weight decay of
$f$ are fixed throughout.

\Cref{FigSyntheticTiltLinearAB}C compares the resulting target-test
MSEs after target-validation tuning.  At zero shift, source ERM,
exact \iw, and \tilt essentially coincide, as they should in the
matched distribution.  As the shift increases, exact \iw deteriorates
rapidly because the ordinary density ratio becomes increasingly
variable.  \tilt remains the most stable method across the positive
shift levels and gives the clearest gains as the shift grows.

\begin{figure}[h!]
  \begin{center}
  \vspace{-4pt}
  \figpanel{0.33\textwidth}{A}{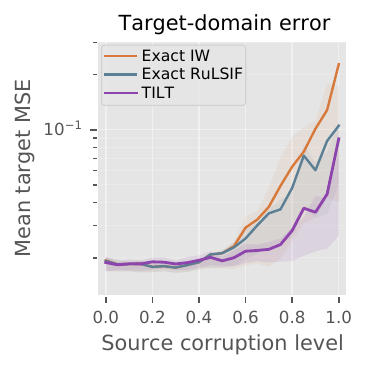}\hfill
  \hspace{-3pt}
  \figpanel{0.33\textwidth}{B}{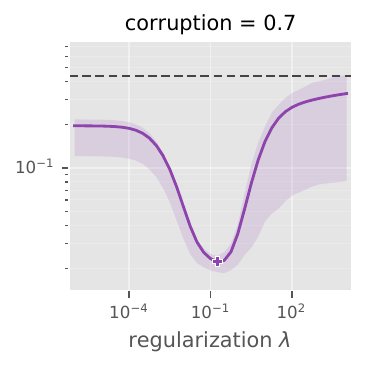}\hfill
  \hspace{-3pt}
  \figpanel{0.33\textwidth}{C}{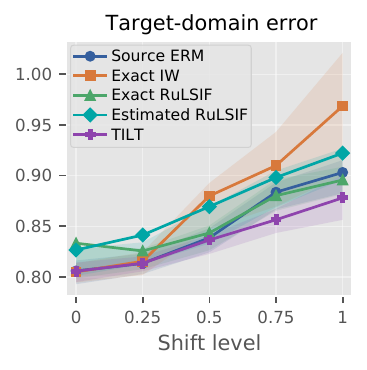}
  \vspace{-4pt}
  \caption{\textbf{\tilt improves target-domain regression under
      covariate shift.}  \textbf{A} In the one-dimensional
    misspecified linear-regression problem, target-test MSE is
    reported as the source distribution moves away from the fixed
    target distribution.  Exact importance weighting (\iw), exact
    relative least-squares importance fitting (\rl), and \tilt are
    shown; source ERM is omitted for scale.  \textbf{B} In the same
    linear setting at corruption level $0.7$, the target-test MSE of
    \tilt is plotted as a function of $\lambda$, with the source ERM
    reference shown for comparison.  \textbf{C} In the
    $4096$-dimensional neural problem, target-test MSE is reported as
    the Gaussian source mean moves away from the fixed target
    distribution.  All methods in \textbf{C} use the same weak ReLU
    target class and the same effective weight decay for the deployed
    predictor $f$; a small labeled target validation set is used only
    to tune allowed hyperparameters.  Curves show means over trials
    and shaded bands show interquartile ranges.}
  \label{FigSyntheticTiltLinearAB}
  \end{center}
  \vspace{-8pt}
\end{figure}

\myparagraph{Minimax point-mass lower-bound experiment} Finally, we
test whether \tilt recovers the effective-sample-size rate predicted
by bounded-density-ratio minimax lower bounds for nonparametric
covariate shift~\citep{pathak2022new}.  We take
$\TargetX=\mathrm{Unif}[0,1]$ and
$\Source_L = L^{-1}\mathrm{Unif}[0,1] + (1-L^{-1})\delta_0$, so
$\|d\qdens/d\pdens_L\|_{\infty}=L$ is tight.  This is a worst-case
shift for the sine-feature model because all sine features vanish at
$x=0$, i.e., $f(0)=0$, making atom samples uninformative about the target function.
The true regression function is
\begin{align*}
  f^*(x)=\sum_{k\geq 1} \theta_k \sqrt{2}\sin(\pi kx),
  \qquad\mbox{with} \quad
  \theta_k \asymp (-1)^{k-1} k^{-(\beta+1/2)},
\end{align*}
with $\beta=2$. Source responses are observed with additive Gaussian
noise, $Y=f^*(X)+\xi$ with $\xi\sim N(0,0.2^2)$, and we report target
MSE against the noiseless regression function $f^*$.  We use the first $D_F\asymp
(n/L)^{1/(2\beta+1)}$ sine features for $\Fclass$, a richer sine span for
$\Bclass\supseteq\Fclass$, and tune both $\Bclass$ and $\lambda$
oracle-wise. In this case, given $d=1$ and $\beta=2.0$, the minimax target-MSE rate is $(n/L)^{-\frac{2\beta}{2\beta+1}}=(n/L)^{-4/5}$~\citep{MaPatWai23}.

\begin{figure}[t]
  \centering
  \includegraphics[width=0.72\textwidth]{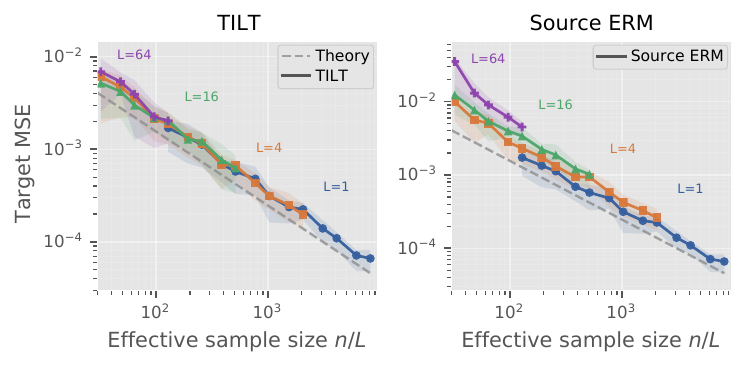}
  \caption{\textbf{Point-mass nonparametric rate.}  Left:
    oracle-tuned \tilt under
    $\Source_L=L^{-1}\mathrm{Unif}[0,1]+(1-L^{-1})\delta_0$ for a
    $\beta=2$ sine-series regression function.  Right: source ERM on
    the same $n/L$ axis ascends as $L$ increases.  The dotted line has slope
    $(n/L)^{-4/5}$.}
  \label{FigSyntheticPointMassNonparametricRate}
\end{figure}

\Cref{FigSyntheticPointMassNonparametricRate} compares \tilt with
source ERM on the same $n/L$ axis.  \tilt aligns with the
$(n/L)^{-4/5}$ theory slope, while source ERM does not align because
it scales as $L n^{-4/5}$ in this construction. Since $L^{4/5}<L$ for
$L>1$, $\tilt$'s minimax scaling is better than the linear dependence
on $L$ of source ERM in the shifted regimes, showing that the \tilt
objective can effectively mitigate the worst-case effect of the large
density ratio in this problem.

\subsection{Covariate shift in CIFAR-100}

\paragraph{Separation between \tilt and \kd}

For CIFAR-100 we use a target-side image corruption intended to mimic
poor acquisition conditions: as the target shift grows, images become
darker, more contrasted, less saturated, and increasingly blurred as
shown in \Cref{FigCifar100}A.  The source domain is clean CIFAR-100.
A ResNet-20 teacher is trained from scratch on the source, and the
tilted variants use a ResNet-20 auxiliary logit model $b$, also
without pretraining.  The deployed target/student model $f$ is an
intentionally weak CNN with about $1.4\times10^4$ parameters, whereas
ResNet-20 has about $2.8\times10^5$ parameters, roughly a $20$-fold
gap.  This makes the student misspecified enough that covariate shift
has a visible effect.  Because $b$ is added inside the logits, we
mean-center the auxiliary logits across classes so that their
class-wise sum is zero and the decomposition does not collapse through
an arbitrary logit shift.

We compare source ERM, vanilla knowledge distillation
(\kd)\cite{hinton2015distilling}, and the two tilted variants \kdtilt
and \kltilt.  The teacher is trained for $160$ epochs and
students for $100$ epochs using SGD with momentum; distillation uses
temperature $T=2.0$ and mixing weight $\beta=0.5$.  Target shift
strength is swept over $\{0,0.33,0.66,1.0,1.33,1.66\}$, omitting the
intermediate $0.83$ run.  The method-comparison curves aggregate ten
seeds, and the $\lambda$-sensitivity curve aggregates seven \kltilt
seeds.  The $\lambda$-sensitivity panel uses separate \kltilt runs at
shift strengths $0$ and $1.5$.  The penalty strength $\lambda$ is
selected per shift strength using labeled target validation data, from
the swept grids for \kdtilt and \kltilt.

\Cref{FigCifar100} shows that a small amount of corruption is enough
for the tilted distillation objectives to separate from the baselines.
In \Cref{FigCifar100}B, \kdtilt and \kltilt are close to the baselines
at small shifts but overtake source ERM and vanilla \kd once the shift
strength becomes moderate.  The target-test cross-entropy in
\Cref{FigCifar100}C makes the same effect more pronounced.  At zero
shift, \kltilt pays a small cost relative to source ERM and vanilla
\kd, while \kdtilt remains comparable, as expected when little
target-side debiasing is needed.  As
the shift grows, both \kdtilt and \kltilt obtain much lower loss than
source ERM and vanilla \kd.  \kdtilt is consistently slightly better
than \kltilt, but \kltilt is also stable and substantially improves
over the non-tilted baselines.  These trends suggest that target-side
debiasing can compensate for a large capacity gap even when the
auxiliary model has the same architecture as the teacher.

\Cref{FigCifar100}D plots the target-test cross-entropy of \kltilt as
a function of $\lambda$ at small and large target shifts.  The finite
curves are relatively flat across a broad intermediate range of
$\lambda$ in both regimes, showing that the method is not sensitive to
precise tuning within that range.  At the same time, performance
degrades or becomes numerically unstable when $\lambda$ is made
extremely small or extremely large, matching the synthetic
experiments: useful target-side debiasing requires an intermediate
amount of regularization rather than either pure importance weighting
or the source-ERM limit.  The dashed source-ERM references show that,
under the larger shift, this robust range remains well below the
source-only loss.

\begin{figure}[h!]
  \vspace{-4pt}
  \begin{center}
  \figpanelbox{0.95\textwidth}{A}{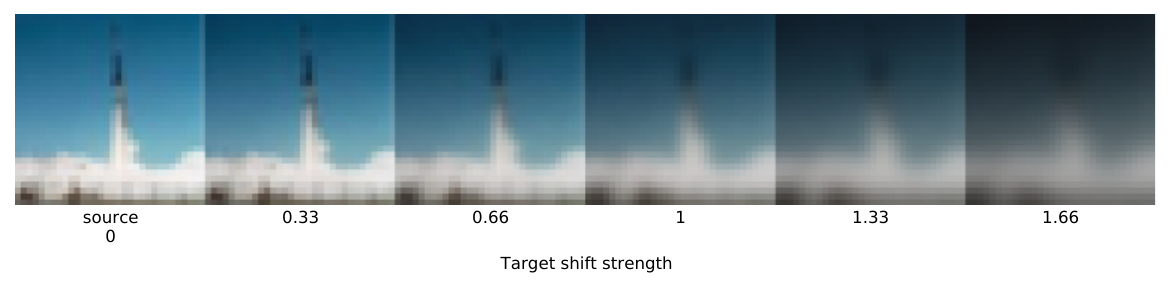}
  \par\vspace{0.35em}
  \vspace{-8pt}
  \figpanel{0.315\textwidth}{B}{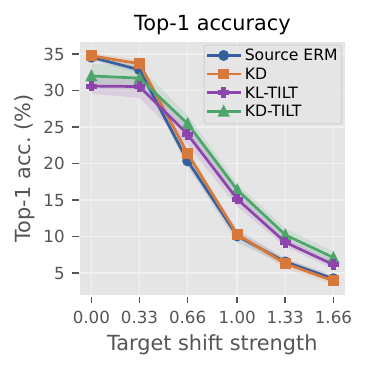}\hfill
  \figpanel{0.315\textwidth}{C}{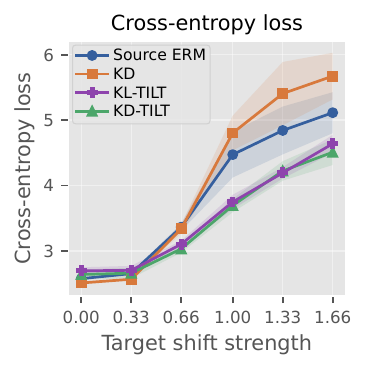}\hfill
  \figpanel{0.315\textwidth}{D}{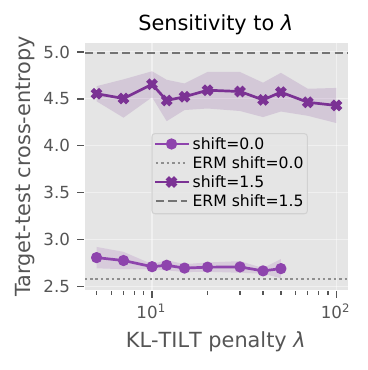}
  \caption{\textbf{\kdtilt and \kltilt consistently improve
      shifted-test performance over source-only baselines on
      CIFAR-100.}  A ResNet-20 teacher is trained on clean CIFAR-100
    (source), and a weak CNN student is distilled under covariate
    shift induced by a deterministic combination of brightness,
    contrast, color, and Gaussian-blur perturbations on the target
    domain, parameterized by the target shift strength on the
    horizontal axis.  \textbf{A} shows the deterministic target-shift
    path for one representative CIFAR-100 test image with the label
    ``rocket''; the leftmost image is the clean source image at shift
    strength $0$.  \textbf{B} reports top-1 accuracy on the shifted
    target-test set; \textbf{C} reports the corresponding target-test
    cross-entropy loss.  \textbf{D} shows the target-test
	    cross-entropy of \kltilt as a function of $\lambda$ at shift
	    strengths $0$ and $1.5$; dashed horizontal lines mark the
	    corresponding source-ERM losses, and missing large-$\lambda$
	    points indicate NaN training runs.  We compare \textbf{Source ERM}
    (student trained with cross-entropy on source only), \textbf{\kd}
    (vanilla knowledge distillation), \textbf{\kdtilt} (KD with the
    \tilt target-side regularizer), and \textbf{\kltilt} (the
    KL-divergence tilted variant).  The penalty strength $\lambda$ for
    \kdtilt and \kltilt is selected per shift strength on a labeled
    target-validation split.  Solid curves show means over seeds and
    shaded bands indicate $\pm1$ standard deviation; \textbf{B} and
    \textbf{C} use ten seeds, while \textbf{D} uses seven seeds.  Both
    \tilt variants clearly dominate Source ERM and \kd at
    moderate-to-high shift strengths, and the \kltilt loss curve
    remains robust over a broad intermediate range of $\lambda$
    values.}
 \label{FigCifar100}
\end{center}
\vspace{-8pt}
\end{figure}
\vspace{-4pt}

\section{Discussion}
\label{SecDiscussionLimitations}
\vspace{-4pt}

We introduced \tilt, a one-step method for unsupervised
covariate-shift adaptation that uses unlabeled target covariates to
regularize an auxiliary component rather than to estimate an explicit
density ratio. The main conceptual point is that a simple additive
decomposition, together with a target-side penalty on the auxiliary
term, induces a profiled objective equivalent to a relative weighted
target excess risk.  Instead of learning a potentially unstable ratio,
the auxiliary component learns a bounded, error-localized correction.

Several directions for future work remain. First, the current
finite-sample theory applies to least-squares regression under
covariate shift, with bounded function classes and sub-Gaussian
noise. The classification variants follow the same algorithmic
principle, but extending the finite-sample guarantees to general
Bregman losses remains to be done.  Second, the method introduces an
auxiliary model and at least one regularization parameter. Although
\tilt avoids explicit density-ratio estimation, it does not eliminate
model-selection issues, including $\lambda$ and the function classes
$(\Fclass, \Bclass)$.  Finally, the empirical evaluation is still
controlled. The synthetic experiments isolate the mechanism, and the
CIFAR-100 shift is generated by deterministic image corruptions. These
experiments show that \tilt can work under substantial covariate
shift, but broader evaluation on natural distribution shifts is needed
to characterize its practical failure modes.

\subsubsection*{Acknowledgements}  This work was partially funded by National Science
Foundation Grant NSF DMS-2311072; Office of Naval Reseach ONR Grant
N00014026-1-2116, and the Ford Professorship to MJW.  KY was supported
by the Takenaka Scholarship Foundation.

\PrintPaperBibliography

\newpage
\appendix

\begin{center}
    \LARGE \textbf{Appendix for ``\tilt: Target-induced loss tilting
      under covariate shift''}
\end{center}

\section{Proof of~\Cref{ThmRiskDecomp}}
\label{SecThmRiskDecomp}

Here we prove the more general decomposition result
\begin{align}
\label{EqnOptimalRisk}  
\Risk(f, b) & = \lambda\Errlam(f) + (1 + \lambda) \|\bfun -
\bstar_f\|_{\Source_\lambda}^2
\end{align}
where $\Source_\lambda$ is the distribution with density $\frac{\pdens
  + \lambda \qdens}{1 + \lambda}$.  Both claims
in~\Cref{ThmRiskDecomp} follow as a direct consequence of this
representation.

By definition, the auxiliary risk is given by
\begin{align*}
  \Risk(f, b) & = \Exs_\Prob \big[(f(X) - \fstar(X) + b(X))^2 \Big] +
  \lambda \Exs_\Qprob[b^2(X)] \; = \; \Exs_\Prob \big[(h(X) + b(X))^2
    \Big] + \lambda \Exs_\Qprob[b^2(X)],
\end{align*}
where we have introduced the shorthand notation $h(x) \defn f(x) -
\fstar(x)$.  Next we write the expectations in terms of the densities
$\pdens$ and $\qdens$, and expand the square, thereby obtaining
\begin{align*}
\Risk(f, \bfun) & = \int \underbrace{\left( h(x)^2 \pdens(x) + 2h(x)
  \bfun(x)\pdens(x) + \bfun^2(x) (\pdens(x) + \lambda \qdens(x))
  \right)}_{\equiv \TmpFun(x)} dx.
\end{align*}
We complete the square with respect to $\bfun(x)$ weighted by
$\pdens(x) \lambda \qdens(x)$. By doing so, the integrand $\TmpFun$
can be decomposed as $\TmpFun(x) = \TmpFun_1(x) + \TmpFun_2(x)$, where
\begin{align*}
\TmpFun_1(x) & \defn \left( \pdens(x) - \frac{\pdens^2(x)}{\pdens(x) +
  \lambda \qdens(x)} \right) h^2(x), \quad \mbox{and} \\
\TmpFun_2(x) & \defn (\pdens(x) + \lambda \qdens(x)) \left( b(x) +
\frac{\pdens(x)}{\pdens(x) + \lambda \qdens(x)} h(x) \right)^2,
\end{align*}
By inspection, we see that $\int \TmpFun_2(x) dx$ is equal to
$(1+\lambda) \|b - \bstar_f\|_{\Source_\lambda}^2$.  As for the term
$\TmpFun_1$, we have
\begin{align*}
\TmpFun_1(x) = \left( \pdens(x) - \frac{(\pdens)^2}{\pdens(x) +
  \lambda \qdens(x)} \right) h^2(x) & = \left( \frac{\pdens(\pdens +
  \lambda \qdens(x)) - \pdens^2(x)}{\pdens(x) + \lambda \qdens(x)}
\right) h^2(x) \\
& = \left(\frac{\lambda \pdens(x) \qdens(x)}{\pdens(x) + \lambda \qdens(x)} \right)
h^2(x) = \lambda \Vlam(x) h^2(x) \qdens(x).
\end{align*}
Consequently, we have $\int \TmpFun_1(x) dx = \lambda
\Exs_\Qprob[\Vlam(X) h^2(X)] = \lambda \Errlam(f)$, which completes
the proof.

\section{Proof of \Cref{ThmExcessRisk}}
\label{AppThmExcessRisk}

We now turn to the proof of~\Cref{ThmExcessRisk}. For the reader's
convenience, we summarize some notation and then provide an overall
roadmap of the proof structure.

\paragraph{Notation:} Throughout the appendix proofs,
we use $\Exs$ to denote expectation over all covariates (both source
and target), as well as over source-response noise. For notational
compactness, we write $\ERisk$ and $\ERiskHat$ for the corresponding
expected population and ideal empirical risks. The ideal empirical
risk can be decomposed as
\begin{align*}
    \RiskHat(f,b) \coloneqq \RiskHat_S(f,b) + \lambda \RiskHat_T(b),
    \qquad \RiskHat_S(f,b) \coloneqq \pempnorm{f+b-\fstar}^2, \qquad
    \RiskHat_T(b) \coloneqq \qempnorm{b}^2.
\end{align*}

\paragraph{Roadmap for \Cref{ThmExcessRisk}}
The proof proceeds in four steps. First, by using the
decomposition~\eqref{EqnOptimalRisk} that underlies the proof
of~\Cref{ThmRiskDecomp}, we convert the population \tilt objective
into the reweighted excess risk plus a nonnegative residual term.
Second, \Cref{LemGeneralization} transfers control from the population
risk to the empirical idealized risk. Third,
\Cref{LemNoiseInteraction} relates this idealized quantity to the
actual empirical minimizer trained with noisy labels. Finally, these
three ingredients are combined and the approximation term is expanded
by minimizing over the auxiliary function.  This yields the main
stochastic inequality. The final step is to rewrite the oracle term
using the explicit minimizer over the auxiliary function.

At a high level, the proof combines the risk decomposition,
generalization control, and noise-interaction control as follows:
\begin{align*}
    \lambda \Exs[\Errlam(\fhat)]
    \underbrace{\leq}_{equation~\eqref{EqnOptimalRisk}} \ERisk(\fhat,
    \bhat) &\underbrace{\lesssim}_{\Cref{LemGeneralization}}
    \ERiskHat(\fhat, \bhat) + \text{Gen. Error}
    \\ &\underbrace{\lesssim}_{\Cref{LemNoiseInteraction}} \inf_{f \in
      \Fclass, b \in \Bclass} \Risk(f,b) + \text{Est. Error} +
    \text{Gen. Error}.
\end{align*}
Here the generalization error and estimation error are made explicit
below in terms of the covering numbers of the function class $\Fclass
+ \Bclass$ on the source samples, and the function class $\Bclass$ on
the target samples.

\subsection{Main argument}
By the Risk Decomposition Identity~\eqref{EqnOptimalRisk}, for any
pair $(\ffun, \bfun)$, we have the upper bound \mbox{$\lambda
  \Errlam(f) \leq \Risk(f, b)$.}  Applying this fact to the pair
$(\fhat, \bhat)$ returned by the \tilt procedure yields
\begin{align}
\label{EqnProofStart}
\lambda \Exs[\Errlam(\fhat)] \leq \ERisk(\fhat, \bhat).
\end{align}
We next need to compare the expected population
risk~\eqref{EqnProofStart} with the ideal empirical risk associated
with the samples. The following lemma provides such a generalization
bound:
\mygraybox{
\begin{lemma}[Generalization bound]
\label{LemGeneralization}
Under the conditions of~\Cref{ThmExcessRisk}, \mbox{for any
  $\epsilon_1\in (0, 1]$,} the learned estimators $(\fhat, \bhat)$
satisfy
\begin{align}
\label{EqProofGen}
    \ERisk(\fhat, \bhat) &\leq (1+\epsilon_1) \ERiskHat(\fhat, \bhat)
    + \frac{c \bou^2}{\epsilon_1} \left( \frac{\log
      \Num_\numsource(\delta; \Fclass + \Bclass)}{\numsource} +
    \lambda \frac{\log \Num_\numtarget(\delta; \Bclass)}{\numtarget}
    \right) + 210 \delta \bou,
\end{align}
where $c$ is a universal constant.
\end{lemma}
}
\noindent
See~\Cref{AppLemGeneralization} for the proof.

Our next step is to control the expected empirical idealized risk
$\ERiskHat(\fhat, \bhat)$. The estimator is fitted with noisy labels,
so the next lemma isolates the resulting noise interaction and
converts the empirical risk into an oracle population quantity.  Let
$\Num_\numsource(\delta; \Fclass + \Bclass)$ be the $\delta$-covering
number of the sum class $\mathcal{H} = \Fclass + \Bclass$ with respect
to the empirical $L_\infty$ norm on the source data.
\mygraybox{
\begin{lemma}[Noise interaction bound]
\label{LemNoiseInteraction}
For any $\delta > 0$ and $\epsilon \in (0,1]$, the learned estimator
$(\fhat, \bhat)$ satisfies
\begin{align}
\label{EqProofNoise}  
    \ERiskHat(\fhat, \bhat) \leq (1 + \epsilon_2) \inf_{f \in \Fclass,
      b \in \Bclass} \Risk(f, b) + \frac{(1 +
      \epsilon_2)^2}{\epsilon_2} \frac{12 \sigma^2 \log
      \Num_\numsource(\delta; \Fclass + \Bclass)}{\numsource} +
    c_\delta,
\end{align}
where $c_\delta = 6 (1 + \epsilon_2) \sigma \delta$.
\end{lemma}
}
\noindent
See~\Cref{AppLemNoiseInteraction} for the proof.

Combining the bound~\eqref{EqProofNoise} with the
bound~\eqref{EqProofGen} yields
\begin{align*}
\ERisk(\fhat, \bhat) &\leq (1+\epsilon_1)(1+\epsilon_2) \inf_{f, b}
\Risk(f, b) + (1+\epsilon_1) \frac{(1+\epsilon_2)^2}{\epsilon_2}
\frac{12 \sigma^2  \log \Num_\numsource}{\numsource} \\
& \quad + \frac{c_1 B^2}{\epsilon_1} \left( \frac{\log \Num_\numsource
  + 1}{\numsource} + \frac{\lambda (\log \Num_\numtarget +
  1)}{\numtarget} \right) + \MyRem(\delta),
\end{align*}
where $\MyRem(\delta) \defn 210 \delta \bou + 6 (1 + \epsilon_2)
\sigma \delta$.

We simplify the coefficients by setting $\epsilon_1 = \epsilon_2 =
\epsilon/3$ for some $\epsilon \in (0, 1]$. Then
  $(1+\epsilon_1)(1+\epsilon_2) = 1 + \epsilon_1 + \epsilon_2 +
  \epsilon_1\epsilon_2 \leq 1 + \epsilon$. The terms proportional to
  $1/\epsilon_1$ and $1/\epsilon_2$ are both
  $O(1/\epsilon)$. Combining the source complexity terms from both
  theorems which scale with $\bou^2/n$ and $\sigma^2/n$ and the target
  complexity term which scales with $\bou^2/m$, we can find a
  sufficiently large constant $c$ such that, using the convention
  $\Num_\numsource,\Num_\numtarget\ge 2$ to absorb additive constants
  into the logarithms,
\begin{align*}
    \ERisk(\fhat, \bhat) \leq (1+\epsilon) \inf_{f, b} \Risk(f, b) +
    \frac{c (\bou^2 + \sigma^2)}{\epsilon} \left( \frac{\log
      \Num_\numsource}{\numsource} + \frac{\lambda \log
      \Num_\numtarget}{\numtarget} \right) + \MyRem(\delta).
\end{align*}
The decomposition of the infimum term follows directly from
equation~\eqref{EqnOptimalRisk}, in particular by substituting the
definition of $\Risk(f, b)$ and minimizing over $b$ for a fixed
$f$. This completes the proof of the theorem.

\subsection{Proof of \Cref{LemGeneralization}}
\label{AppLemGeneralization}

The excess auxiliary risk can be decomposed as $\Risk(\ffun, \bfun) =
\Risk_S(\ffun, \bfun) + \lambda \Risk_T(\bfun)$, where
\mbox{$\Risk_S(\ffun, \bfun) \defn \|\ffun + \bfun -
  \fstar\|_\Source^2$,} and $\Risk_T(\bfun) \defn
\|\bfun\|_\Target^2$.  Using this notation, we have the decomposition
\begin{align*}
\left| \ERisk(\fhat, \bhat) - \ERiskHat(\fhat, \bhat) \right| & \leq
\underbrace{\left|\ERisk_S(\fhat, \bhat) - \ERiskHat_S(\fhat,
  \bhat)\right|}_{\equiv \Term_1} + \lambda
\underbrace{\left|\ERisk_T(\bhat) - \ERiskHat_T(\bhat)\right|}_{\equiv
  \Term_2}
\end{align*}
Note that functions of the form $f + \bfun$ are $3 \bou$-uniformly
bounded, whereas the function $\bfun$ is $\bou$-uniformly bounded.
Thus, we can apply~\Cref{LemConcentration} to control each term.

Beginning with $\Term_1$, setting $B_{\Hil} = 3 \bou$ yields
\begin{subequations}
\begin{align}
\label{EqnTermOne}    
\Term_1 & \leq \frac{3 \bou}{\sqrt{\numsource}} \sqrt{\ERisk_S(\fhat,
  \bhat)} \sqrt{36 \log \Num_\numsource + 256} + 9 \bou^2 \frac{6 \log
  \Num_\numsource + 11}{\numsource} + 78 \delta \bou.
    \end{align}
As for $\Term_2$, setting $B_{\Hil} = \bou$ yields
\begin{align}
\label{EqnTermTwo}  
\Term_2 & \leq \frac{\bou}{\sqrt{\numtarget}} \sqrt{\ERisk_T(\bhat)}
\sqrt{36 \log \Num_\numtarget + 256} + B^2 \frac{6 \log
  \Num_\numtarget + 11}{\numtarget} + 26 \delta \bou.
\end{align}
In order to combine these two bounds, we use the Cauchy--Schwarz
inequality in the form
\begin{align}
\label{EqnCS}  
  \sqrt{u} s + \sqrt{\lambda v} (\sqrt{\lambda} t) \leq \sqrt{u +
    \lambda v} \sqrt{s^2 + \lambda t^2},
\end{align}
\end{subequations}

With the choices $u = \ERisk_S(\fhat, \bhat)$ and $v =
\ERisk_T(\bhat)$, we have $u + \lambda v = \ERisk(\fhat, \bhat)$.  We
then set
\[
s^2 = \frac{9B^2 (36 \log \Num_\numsource + 256)}{\numsource}
\qquad\text{and}\qquad
t^2 = \frac{B^2(36 \log \Num_\numtarget + 256)}{\numtarget}.
\]
Applying inequality~\eqref{EqnCS} and combining with the
bounds~\eqref{EqnTermOne} and~\eqref{EqnTermTwo} then yields
\begin{align*}
    \left|\ERisk(\fhat, \bhat) - \ERiskHat(\fhat, \bhat)\right|
    &\leq
    \sqrt{\ERisk(\fhat, \bhat)}
    \sqrt{
        \frac{9B^2(36 \log \Num_\numsource + 256)}{\numsource}
        + \lambda
        \frac{B^2(36 \log \Num_\numtarget + 256)}{\numtarget}
    } \\
    &\quad
    + 9B^2
    \frac{6 \log \Num_\numsource + 11}{\numsource}
    + \lambda B^2
    \frac{6 \log \Num_\numtarget + 11}{\numtarget}
    + 104 \delta B .
\end{align*}
We can write this inequality more compactly in the form $|s - t| \leq
2\sqrt{su} + v$, where $s \defn \ERisk(\fhat, \bhat)$, $t \defn
\ERiskHat(\fhat, \bhat)$, and
\begin{align*}
  u & \defn \frac{1}{4} \left( \frac{9B^2(36 \log \Num_\numsource +
    256)}{\numsource} + \lambda \frac{B^2(36 \log \Num_\numtarget +
    256)}{\numtarget} \right), \\
v & \defn 9B^2 \frac{6 \log \Num_\numsource + 11}{\numsource} +
\lambda B^2 \frac{6 \log \Num_\numtarget + 11}{\numtarget} + 104
\delta B .
\end{align*}
Now our inequality implies that $s \leq t + 2 \sqrt{s u} + v \; \leq
\; t + \eta s + \frac{u}{\eta} + v$, where the second inequality
follows by Young's inequality, valid for any $\eta \in (0,1)$.
Re-arranging yields
\begin{align*}
s & \leq \frac{t}{1 - \eta} + \frac{v}{1-\eta} + \frac{u}{ \eta \,(1 -
  \eta)}.
\end{align*}
We then choose $\eta = \frac{\epsilon}{1 + \epsilon}$ for some
$\epsilon \in (0,1)$, so that $\frac{1}{1 - \eta} = 1 + \epsilon$,
thereby obtaining the bound
\begin{align*}
s & \leq (1+\epsilon)t + (1+\epsilon)v +
\frac{(1+\epsilon)^2}{\epsilon}u \; \leq \; (1+\epsilon)t +
2 v
+ \frac{4}{\epsilon} u,
\end{align*}
where the second step follows since $1 + \epsilon \leq 2$.  Recalling
our definitions of $(u, v)$, we have
\begin{align*}
2 v + \frac{4}{\epsilon} u &\leq 18B^2 \frac{6 \log \Num_\numsource +
  11}{\numsource} + 2\lambda B^2 \frac{6 \log \Num_\numtarget +
  11}{\numtarget} + 208\delta B \\ &\quad + \frac{1}{\epsilon} \left(
\frac{9B^2(36 \log \Num_\numsource + 256)}{\numsource} + \lambda
\frac{B^2(36 \log \Num_\numtarget + 256)}{\numtarget} \right).
\end{align*}
Since $1/\epsilon \geq 1$, we have shown that
\begin{align*}
  \ERisk(\fhat, \bhat) \leq (1+\epsilon)\ERiskHat(\fhat, \bhat) +
  \frac{B^2}{\epsilon} \left( \frac{c_1 \log \Num_\numsource +
    c_2}{\numsource} + \lambda \frac{c_3 \log \Num_\numtarget +
    c_4}{\numtarget} \right) + 208 \delta B .
\end{align*}
The asserted bound follows after enlarging the universal constants.

\subsection{Proof of \Cref{LemNoiseInteraction}}
\label{AppLemNoiseInteraction}

We begin by relating the empirical objective minimized by the
estimator to the ideal risk. Recall that our estimator $(\fhat,
\bhat)$ minimizes the empirical objective $\LossHat(f, b)$.  The
source component involves the pairs $(x_i, y_i)$, and we can write
$y_i = \fstar(x_i) + \xi_i$, where $\fstar(x) = \Exs[Y \mid X = x]$,
and $\xi_i$ is conditionally zero-mean noise.

Let $(f, \bfun)$ be any fixed pair of functions in $\Fclass \times
\Bclass$. By the optimality of $(\fhat, \bhat)$, we have the basic
inequality $\LossHat(\fhat, \bhat) \leq \LossHat(f, \bfun)$.  We now
expand the squared loss term in $\LossHat$. For any functions $h$, the
empirical source loss is given by
\begin{align*}
\frac{1}{\numsource} \sum_{i=1}^\numsource (h(x_i) - y_i)^2 &=
\frac{1}{\numsource} \sum_{i=1}^\numsource (h(x_i) - \fstar(x_i) -
\xi_i)^2 \\
& = \frac{1}{\numsource} \sum_{i=1}^\numsource (h(x_i) -
\fstar(x_i))^2 - \frac{2}{\numsource} \sum_{i=1}^\numsource \xi_i
(h(x_i) - \fstar(x_i)) + \frac{1}{\numsource} \sum_{i=1}^\numsource
\xi_i^2.
\end{align*}
Substituting this expansion into the inequality $\LossHat(\fhat,
\bhat) \leq \LossHat(f, b)$ and re-arranging, we find that
\begin{align*}
  \pempnorm{\fhat + \bhat - \fstar}^2 + \lambda \qempnorm{\bhat}^2
  \leq \pempnorm{f + b - \fstar}^2 + \lambda \qempnorm{b}^2 +
  \frac{2}{\numsource} \sum_{i=1}^\numsource \xi_i (\fhat(x_i) -
  f(x_i) + \bhat(x_i) - b(x_i)).
\end{align*}
We now take the expectation over the dataset realizations. Recall that
$\ERiskHat(\fhat, \bhat) = \Exs[\pempnorm{\fhat + \bhat - \fstar}^2 +
  \lambda \qempnorm{\bhat}^2]$. For the fixed functions $f, \bfun$ on
the right-hand side, the expectation of the noise interaction term is
zero because $\Exs[\xi_i | x_i] = 0$ and $f, \bfun$ are independent of
$\xi$. Furthermore, $\Exs[\pempnorm{f + \bfun - \fstar}^2 + \lambda
  \qempnorm{\bfun}^2] = \Risk(f, \bfun)$. Using these facts, we arrive
at the inequality
\begin{align}
\label{EqNoiseInequalityStart}  
\ERiskHat(\fhat, \bhat) \leq \Risk(f, b) + \Exs \left[
  \frac{2}{\numsource} \sum_{i=1}^\numsource \xi_i (\fhat(x_i) +
  \bhat(x_i) - \fstar(x_i)) \right].
\end{align}
The core difficulty lies in bounding the expected noise interaction
term on the right-hand side, as $\fhat$ and $\bhat$ depend on the
noise $\xi$. We do so using a covering number argument.

Let $\Hclass = \{ f + \bfun - \fstar \mid f \in \Fclass, \bfun \in
\Bclass \}$. Let $\hat{h} = \fhat + \bhat - \fstar \in \Hclass$. We
seek to bound the random variable
\begin{align*}
Z \defn \Exs [ \frac{2}{\numsource} \sum_{i=1}^\numsource \xi_i
  \hat{h}(x_i) ].
\end{align*}
Consider a minimal $\delta$-covering of $\Hclass$ with respect to the
empirical $L_\infty$ norm, denoted by $\{h_1, \dots, h_N\}$, where $N
= \Num_\numsource(\delta; \Fclass + \Bclass)$. By construction there
exists a (random) index $j^*$ such that $\|\hat{h} - h_{j^*}\|_\infty
\leq \delta$.  Using this fact, we can decompose $Z$
as
\begin{align*}
Z = \frac{2}{\numsource} \sum_{i=1}^\numsource \xi_i (\hat{h}(x_i) -
h_{j^*}(x_i)) + \frac{2}{\numsource} \sum_{i=1}^\numsource \xi_i
h_{j^*}(x_i).
\end{align*}
For the first term, by Hölder's inequality and the bound on the cover,
we have $|\frac{2}{\numsource} \sum \xi_i (\hat{h} - h_{j^*})| \leq
\frac{2}{\numsource} \sum |\xi_i| \delta$. Taking expectations, this
is bounded by $2 \delta \Exs[|\xi|] \leq 2\delta \sigma$ using
Jensen's inequality.

As for the second term, we let $\zeta_{j} = \frac{\sum_i \xi_i
  h_j(x_i)}{\sqrt{\numsource} \|h_j\|_n}$ be a normalized noise
interaction for the fixed function $h_j$. Conditioned on $\{x_i\}$,
the auxiliary r.v. $\zeta_{j}$ is a weighted sum of independent
sub-Gaussian variables, which is itself sub-Gaussian with variance
proxy $\sigma^2$. We can rewrite the second term as:
\begin{align*}
\Exs \left[ \frac{2}{\numsource} \sum_{i=1}^\numsource \xi_i
  h_{j^*}(x_i) \right] \leq \Exs
  \left| \frac{2}{\sqrt{\numsource}} \|h_{j^*}\|_n \zeta_{j^*}
  \right|
\end{align*}
Applying the Cauchy--Schwarz inequality yields
\begin{align*}
\Exs \left[\frac{2}{\sqrt{\numsource}} \|h_{j^*}\|_n |\zeta_{j^*}|
  \right] \leq \frac{2}{\sqrt{\numsource}} \left(
\Exs[\|h_{j^*}\|_n^2] \right)^{1/2} \left( \Exs[\max_{j\in[N]}
  \zeta_j^2] \right)^{1/2}.
\end{align*}
By standard results on the maxima of sub-Gaussian variables, we have
\begin{align*}
\Exs[\max_{j=1, \ldots, N} \zeta_j^2] \leq \sigma^2 (3 \log N + 1).
\end{align*}
For the norm term, we have $\|h_{j^*}\|_n \leq \|\hat{h}\|_n + \delta$
follows from the definition of $h_{j^*}$ and the triangle
inequality. Also, the triangle inequality in $L_2(\Exs[\cdot])$ yields
that $\Exs[\|h_{j^*}\|_n^2]^{1/2} \leq \Exs[\|\hat{h}\|_n^2]^{1/2} +
\delta$. Remaining $\RiskHat(\fhat, \bhat)=\|\hat{h}\|_n^2 +
\lambda\|\bhat\|_m^2$, so we have $\Exs[\|h_{j^*}\|_n^2]^{1/2} \leq
\sqrt{\ERiskHat(\fhat, \bhat)} + \delta$.

Putting together the pieces, we have established the upper bound
\begin{align*}
\Exs[Z] & \leq \frac{2\sigma \sqrt{3 \log N + 1}}{\sqrt{\numsource}}
\left( \sqrt{\ERiskHat(\fhat, \bhat)} + \delta \right) + 2\delta
\sigma.
\end{align*}
Substituting this upper bound back into
equation~\eqref{EqNoiseInequalityStart}, we find that
\begin{align*}
\ERiskHat(\fhat, \bhat) & \leq \Risk(f, b) + 2\sqrt{\ERiskHat(\fhat,
  \bhat)} \sqrt{\frac{\sigma^2(3 \log N + 1)}{\numsource}} + \delta
\left( 2\sigma \sqrt{\frac{3 \log N + 1}{\numsource}} + 2\sigma
\right) \\
& \leq \Risk(f, b) + 2\sqrt{\ERiskHat(\fhat,
  \bhat)} \sqrt{\frac{\sigma^2(3 \log N + 1)}{\numsource}} + 6 \delta \sigma,
\end{align*}
where the second inequality follows from our assumption that $\log N
\leq n$.  Note that this inequality is of the form $u \leq v + 2
\sqrt{u s} + t$ where $u = \ERiskHat(\fhat, \bhat)$, $v = \Risk(\ffun,
\bfun)$, $s = \sqrt{\frac{\sigma^2(3 \log N + 1)}{\numsource}}$ and $t
= 6 \delta \sigma$.  By Young's inequality, it implies that $u \leq
\frac{v}{1-\eta} + \frac{s}{\eta \, (1 - \eta)} + \frac{t}{1 - \eta}$
for any $\eta \in (0,1)$.  We can choose $\eta \in (0,1)$ such that
$\frac{1}{1 - \eta} = 1 + \epsilon$ for an arbitrary $\epsilon \in
(0,1)$.  Doing so yields
\begin{align*}
\ERiskHat(\fhat, \bhat) \leq (1+\epsilon) \Risk(f, b) +
\frac{(1+\epsilon)^2}{\epsilon} \frac{\sigma^2(3 \log N +
  1)}{\numsource} + (1+\epsilon) 6\sigma\delta.
\end{align*}
Since this holds for any fixed $(f, b)$, we can take the infimum over
the function classes, thereby obtaining the claimed result.

\section{Proof of \Cref{CorExcessRisk}}
\label{AppCorExcessRisk}

We now turn to the proof of~\Cref{CorExcessRisk}.

\subsection{Sparse ReLU approximation tools}
\label{SecPreliminaryReLU}

This section collects the analytic tools used to derive
\Cref{CorExcessRisk} from the oracle inequality in
\Cref{ThmExcessRisk}. We first recall the H\"older balls and sparse
clipped ReLU classes used in the main text. We then state the joint
approximation lemma showing that the auxiliary class can approximate
the population correction $b_f^*(x)=-\Vlam(x)(f(x)-\fstar(x))$ while
retaining explicit control of depth, sparsity, and multiplication
precision.

Recall that our result imposes H\"{o}lder smoothness conditions on the
regression function $\fstar$ and the offset weight $\Vlam$.
\begin{definition}[H\"older space]
\label{DefHolder}
For a domain $\Xspace \subset \mathbb{R}^d$, smoothness index $\beta >
0$, and radius $K > 0$, the H\"older ball $C_d^\beta(\Xspace, K)$
corresponds to the set of all functions $f: \Xspace \to \mathbb{R}$
such that
\begin{align*}
\sum_{\boldsymbol{\alpha}: |\boldsymbol{\alpha}| < \beta}
\|\partial^{\boldsymbol{\alpha}} f\|_\infty +
\sum_{\boldsymbol{\alpha}: |\boldsymbol{\alpha}| = \lfloor \beta
  \rfloor} \sup_{\substack{\mathbf{x}, \mathbf{y} \in D \\ \mathbf{x}
    \neq \mathbf{y}}} \frac{|\partial^{\boldsymbol{\alpha}}
  f(\mathbf{x}) - \partial^{\boldsymbol{\alpha}}
  f(\mathbf{y})|}{\|\mathbf{x} - \mathbf{y}\|_\infty^{\beta - \lfloor
    \beta \rfloor}} \leq K,
\end{align*}
where $\boldsymbol{\alpha} = (\alpha_1, \dots, \alpha_d) \in
\mathbb{N}_0^d$ is a multi-index, $|\boldsymbol{\alpha}| = \sum_i
\alpha_i$, and $\partial^{\boldsymbol{\alpha}} = \partial_1^{\alpha_1}
\dots \partial_d^{\alpha_d}$.
\end{definition}

Our result applies to the \tilt procedure applied using sparse ReLU
neural networks, along with a clipping function that enforces bounded
outputs.  We use $\sigma(x) = \max(x, 0)$ to denote the ReLU
activation function, For a shift vector $\mathbf{v} \in \mathbb{R}^r$,
let $\sigma_{\mathbf{v}}: \mathbb{R}^r \to \mathbb{R}^r$ be the
component-wise map $\sigma_{\mathbf{v}}(y_1, \dots, y_r) =
(\sigma(y_1-v_1), \dots, \sigma(y_r-v_r))^\top$.

\begin{definition}[Sparse deep ReLU networks]
\label{DefNN}
A ReLU network with architecture $(L, \mathbf{p})$ is a function of
the form
\begin{align*}
    f(\mathbf{x}) = W_L \sigma_{\mathbf{v}_L} W_{L-1}
    \sigma_{\mathbf{v}_{L-1}} \dots W_1 \sigma_{\mathbf{v}_1} W_0
    \mathbf{x},
\end{align*}
where $W_i \in \mathbb{R}^{p_{i+1} \times p_i}$ are weight matrices
and $\mathbf{v}_i \in \mathbb{R}^{p_i}$, $i=1,\dots,L$, are shift
vectors. We define the class $\NNs(L, \mathbf{p}, s, B)$ as the set of
such functions satisfying:
\begin{enumerate}[label=(\roman*)]
    \item $L$ hidden layers and width vector $\mathbf{p} = (p_0,
      \dots, p_{L+1})$.
    \item Sparsity constraint: $\sum_{j=0}^L \|W_j\|_0 + \sum_{j=1}^L
      \|\mathbf{v}_j\|_0 \leq s$.
    \item Parameter bound: $\max_{0 \le j \le L} \|W_j\|_\infty \lor
      \max_{1 \le j \le L} \|\mathbf{v}_j\|_\infty \leq 1$.
    \item Clipping: $\|f\|_\infty \leq B$.
\end{enumerate}
\end{definition}

The next lemma combines the H\"older approximation theorem for sparse
ReLU networks with a ReLU multiplication sub-network of size $M$.
This is required in analysis, since ReLU networks generate piecewise
affine outputs, and so can approximate a product $u v$ only up to some
finite precision.  A multiplication sub-network yields an
approximation accurate to the order $2^{-M}$.  In addition, we use
$c_j, j = 0, 1, 2$ etc. to denote universal constants.
\mygraybox{
\begin{lemma}[Joint approximation by sparse ReLU networks]
\label{LemJointApprox}
Consider the sparse ReLU classes $\Fclass = \NNs(L_f, \mathbf{p}_f,
s_f, B)$ and $\Bclass = \NNs(L_g, \mathbf{p}_g, s_g, B)$, and suppose
that $\fstar\in C_d^\beta([0,1]^d,K)$ and $\Vlam \in
C_d^\gamma([0,1]^d, K)$.  Then we have:
\begin{enumerate}[label=(\roman*)]
\item The regression function $\fstar$ satisfies the approximation
  bound
  \begin{subequations}
\begin{align}  
  \inf_{f \in \Fclass} \Errlam(f) \equiv \inf_{f \in \Fclass}
  \Exs_\Qprob \Big[ \Vlam(X) \big(f(X) - \fstar(X) \big)^2 \Big] &
  \leq c_0 K^2 s_f^{-2\beta/d} \log^2(s_f).
\end{align}
\item For any $f \in \Fclass$, any sparsity level $s_v \ge 1$, and any
  multiplication precision $M = 1, 2, \ldots$, there exists a function
  $\bfun \in \Bclass$ such that for any architecture with
  \begin{align}
\label{EqnLayerBound}    
    L_g \geq c_1 \big \{ L_f + \log s_f + \log s_v + M \big \}, \;
    \mbox{and} \; s_g \geq c_2 \big \{ s_f + s_v + M \log M \big \},
\end{align}
we have
\begin{align}
  \| b - b^*_f \|_{\Source_\lambda}^2 \leq c_3 \left(
  s_f^{-2\beta/d}\log^2 s_f + s_v^{-2\gamma/d}\log^2 s_v + 2^{-2M}
  \right),
\end{align}
\end{subequations}
where the constants may depend on the fixed problem parameters
$(d,\beta,\gamma,B,K)$, but not on $s_f,s_v$ or $M$.
\end{enumerate}
\end{lemma}
}
\noindent
See \Cref{AppLemJointApprox} for the proof.

\subsection{Main argument}

We now turn to the proof of~\Cref{CorExcessRisk}.  Throughout, we use
$(c, c', c_j)$ etc. to denote constants independent of the sample
sizes and network sizes, while allowing dependence on the fixed
smoothness, boundedness, and dimension parameters, including the
H\"older radius $K$.  When this radius is uniform in $\lambda$, the
only displayed $\lambda$-dependence is through $\kaplam$ and $\nlam$.
From the statement of the corollary, recall the shorthand notation $\nlam \defn \min
\big \{ \lambda\numsource, \numtarget \big \}$ and $\kaplam \defn (1 +
\lambda)/\lambda$.  We fix the oracle parameter in
\Cref{ThmExcessRisk}, say $s=1/2$.  We choose $s_v\asymp (\kaplam
\nlam)^{d/(2\gamma+d)}$ for the approximation of $\Vlam$, and
choose the multiplication precision $M\asymp \log(\kaplam\nlam)$. The
network sizes in \Cref{CorExcessRisk} ensure the
bounds~\eqref{EqnLayerBound} hold, so we can
apply~\Cref{LemJointApprox}.  For the (fixed) triple $\beta,\gamma,d$,
the $M\log M$ term is absorbed by the polynomial sparsity terms, and
the depth requirement is of order $\log(\kaplam \nlam)$.  For
$\nlam \geq 2$, choose $\delta=(\kaplam\nlam)^{-A}$ with a sufficiently
large constant $A$.  For the fixed value of $\lambda$ under
consideration, the remainder term $\MyRem(\delta)$
in~\Cref{ThmExcessRisk} is then of lower order than the rate below,
while the additional
$\log(1/\delta)$ factor is absorbed into the displayed logarithmic
terms. From the sparse ReLU entropy bound~\citep[Lemma
  A.1]{schmidt2020nonparametric} and using the polynomial width
condition, we have (up to logarithmic factors) an upper bound of the
form
\begin{align*}
    \frac{\log
      \Num_\numsource(\delta;\Fclass+\Bclass)}{\lambda\numsource} +
    \frac{\log \Num_\numtarget(\delta;\Bclass)}{\numtarget} \lesssim
    \frac{s_f+s_g}{\nlam}\,\big(\log(\kaplam \nlam)\big)^2 .
\end{align*}
With the auxiliary choices above, this quantity can be further upper bounded by
\begin{align*}
\frac{s_f+s_v}{\nlam}\,\big(\log(\kaplam \nlam)\big)^3 .
\end{align*}
For the approximation term, \Cref{ThmRiskDecomp}
and~\Cref{LemJointApprox} imply that
\begin{align*}
  \frac{1}{\lambda} \inf_{f\in\Fclass,b\in\Bclass}\Risk(f,b) \lesssim
  s_f^{-2\beta/d}\log^2 s_f + \kaplam \left( s_f^{-2\beta/d}\log^2 s_f
  +s_v^{-2\gamma/d}\log^2 s_v +2^{-2M} \right).
\end{align*}
Taking $M \asymp \log(\kaplam\nlam)$ makes the multiplication error
negligible relative to the polynomial terms. Thus \Cref{ThmExcessRisk}
is controlled, up to logarithmic factors, by
\begin{align*}
    (1+\kaplam)s_f^{-2\beta/d} +\kaplam s_v^{-2\gamma/d}
  +\frac{s_f+s_v}{\nlam}.
\end{align*}
Since $\kaplam > 1$, we can absorb $1+\kaplam$ into $\kaplam$.
Balancing the first and third terms yields $s_f \asymp (\kaplam
\nlam)^{d/(2\beta+d)}$, while balancing the second and third
yields $s_v\asymp (\kaplam \nlam)^{d/(2\gamma+d)}$. Substituting
these proof choices gives
\begin{align*}
    \Exs\Errlam(\fhat) \lesssim \kaplam^{\frac{d}{2\beta+d}}
    \nlam^{-\frac{2\beta}{2\beta+d}} \big(\log(\kaplam \nlam)\big)^3 +
    \kaplam^{\frac{d}{2\gamma+d}}
    \nlam^{-\frac{2\gamma}{2\gamma+d}} \big(\log(\kaplam \nlam)\big)^3 ,
\end{align*}
as claimed.

\subsection{Proof of \Cref{LemJointApprox}}
\label{AppLemJointApprox}

Part (i) follows directly from Theorem 5 in the
paper~\cite{schmidt2020nonparametric}, so that it remains to prove
part (ii).

Given an arbitrary $f \in \Fclass$, our goal is to approximate the
optimal offset $\bstar_f(x) \defn -\Vlam(x)(f(x)-\fstar(x))$, where
$\Vlam \in C_d^\gamma$ by assumption.  Fix a sparsity index $s \geq 1$
and multiplication parameter $M \geq 1$.  Since $\Vlam \in C_d^\gamma$
and $\fstar \in C_d^\beta$, Theorem 5 in the
paper~\cite{schmidt2020nonparametric} guarantees the existence of
networks $\widehat{v}_\lambda\in\NNs(L_v,\mathbf{p}_v,s_v,1)$ and
$\hat{\fstar}\in\NNs(L_{f^*},\mathbf{p}_{f^*},s_f,B)$ such that
\begin{align*}
  \|\widehat{v}_\lambda-\Vlam\|_\infty^2 \lesssim
  s_v^{-2\gamma/d}\log^2 s_v, \qquad \|\hat{\fstar}-\fstar\|_\infty^2
  \lesssim s_f^{-2\beta/d}\log^2 s_f .
\end{align*}
We construct $\bfun$ using the multiplication network
$\mathrm{Mult}_M$ from Lemma A.2 in the
paper~\cite{schmidt2020nonparametric}. This network approximates the
product $uv$ on $[0,1]^2$ with error $2^{-M}$; by an affine rescaling,
the same construction approximates products on $[0,1]\times[-1,1]$
with the same order of error. After the standard clipping step, we may
take $\widehat{v}_\lambda\in[0,1]$. Since $|f-\hat{\fstar}|\le 2B$,
define
\begin{align*}
    b(x) = -2 B\, \mathrm{Mult}_M \left( \widehat{v}_\lambda(x),
    \frac{f(x)-\hat{\fstar}(x)}{2B} \right).
\end{align*}
By the standard parallelization and composition rules, this network is
contained in $\Bclass$ whenever the pair $(L_g, s_g)$ satisfies
the lower bounds~\eqref{EqnLayerBound}.

For the error analysis, let $\Delta_f(x)=f(x)-\fstar(x)$ and
$\hat{\Delta}_f(x)=f(x)-\hat{\fstar}(x)$. By the triangle inequality, we have

\begin{align*}
|b(x)-\bstar_f(x)| &\leq b(x)+\widehat{v}_\lambda(x)\hat{\Delta}_f(x)|
|+ \widehat{v}_\lambda(x)\hat{\Delta}_f(x)-\Vlam(x)\Delta_f(x)|\\
& \lesssim 2B\,2^{-M} + \|\widehat{v}_\lambda-\Vlam\|_\infty \,
|f(x)-\fstar(x)| + \|\hat{\fstar}-\fstar\|_\infty .
\end{align*}
The functions in $\Fclass$ are bounded by $B$, and
$\|\fstar\|_\infty\le K$ by the H\"older condition on $\fstar$.
Integrating with respect to $\Source_\lambda$ and using the upper
bound $(a+b+c)^2 \leq 3(a^2+b^2+c^2)$ yields
\begin{align*}
    \|b- \bstar_f\|_{\Source_\lambda}^2 \lesssim 2^{-2M}
    + s_v^{-2\gamma/d} \log^2 s_v + s_f^{-2\beta/d} \log^2 s_f .
\end{align*}
This proves the claimed bound.

\subsection{Auxiliary lemmas}

Here we restate an auxiliary lemma from the
paper~\cite{schmidt2020nonparametric} (cf. equation (42)).
\mygraybox{
\begin{lemma}[Concentration of $L_2$ risk]
\label{LemConcentration}
Let $S=(x_1,\dots,x_\numsource)$ be an iid sample from $P$, and let
$\Exs_S$ denote expectation over $S$. Let $\Hil$ be a class of
functions bounded by $B_{\Hil}$, and let $h=h_S$ be a measurable
sample-dependent element of $\Hil$. Define $\|h_S\|_P^2 = \Exs_{X \sim
  P}[h_S(X)^2]$ and $\|h_S\|_n^2 =
\frac{1}{\numsource}\sum_{i=1}^\numsource h_S(x_i)^2$. Then
\begin{align*}
    \left| \Exs_S[\|h_S\|_P^2] - \Exs_S[\|h_S\|_n^2] \right| &\leq
    \frac{B_{\Hil}}{\numsource} \sqrt{\Exs_S[\|h_S\|_P^2]}
    \sqrt{36\numsource \log \Num_\numsource(\delta; \Hil) + 2^8
      \numsource} \notag\\ &\quad + B_{\Hil}^2 \frac{6 \log
      \Num_\numsource(\delta; \Hil) + 11}{\numsource} + 26 \delta
    B_{\Hil},
\end{align*}
where $\Num_\numsource(\delta; \Hil)$ is an almost-sure deterministic
upper bound on the $\delta$-covering number of $\Hil$ with respect to
the empirical $L_\infty$-norm induced by $S$.
\end{lemma}
}

\section{Generalization to Bregman divergences}
\label{SecGeneralization}

The construction in~\Cref{SecMethods} is the least-squares
instantiation of a more general Bregman-based method.  In this
section, we describe this more general theory, and its instantiations
for $K$-ary classification problems, as described
in~\Cref{SecExtensions}.

\subsection{Bregman formulation and decomposition}

Let $\realstar \defn \real \cup \{\infty \}$ denote the extended real
line.  Any Bregman divergence is defined by a function $\psi : \Omega
\to \realstar$ that is strictly convex, continuously differentiable,
and of Legendre type~\citep{rockafellar1970convex}.  The function
$\psi$ has a Fenchel conjugate $\psi^*: \Omega^* \to \realstar$, and
the gradient mapping $\nabla \psi$ maps between the interior of
$\Omega$ and $\Omega^*$ in a one-to-one way.

The simplest example is the quadratic function $\psi(z) = \frac{1}{2}
\|z\|_2^2$ with domain $\Omega = \real^K$.  It is self-dual with
Fenchel conjugate $\psi^* = \psi$ and domain $\Omega^* = \real^K$.
For $K$-aray classification problems, the logistic function $\psi(z) =
\log \big(\sum_{k=1}^K e^{z_k} \big)$ arises in the standard loss
function.  It has domain $\Omega = \real^K$ and Fenchel conjugate
$\psi^*$ given by the entropy function defined on the simplex
$\Omega^* = \{ u \in \real^K \mid u \geq 0, \sum_{j=1}^k u_j = 1 \}$.

Given a prediction function $f : \Xspace \to \Omega$, we define the
mean function associated with $f$ as
\begin{subequations}
\begin{align}
\mu_f(x) \coloneqq \nabla \psi(f(x)) \in \Omega^*,
\end{align}
along with the optimal target in mean space
\begin{align}
\mu^\star(x) \coloneqq \Exs[Y \mid X = x] \in \Omega^*,
\end{align}
The canonical pointwise loss that compares this pair is given by
\begin{align}
\bar{\ell}_\psi(f;x) \coloneqq D_{\psi^*}\bigl(\mu^\star(x),
\mu_f(x)\bigr).
\end{align}
Let $\pdens$ and $\qdens$ denote the source and target densities, and
set
\begin{align}
\label{EqnDefnRho}  
\Vlam(x) \coloneqq \frac{\pdens(x)}{\pdens(x)+\lambda\qdens(x)},
\qquad d\rho_\lambda(x) \coloneqq
\bigl(\pdens(x)+\lambda\qdens(x)\bigr)dx.
\end{align}        
Finally, for a convex function $\phi$ and a scalar $\eta \in [0,1]$,
we define the weighted Jensen divergence
\begin{align}
\label{EqnJensen}  
J_\phi^\eta(u,v) & \defn \eta \phi(u) + (1 - \eta) \phi(v) - \phi
\bigl(\eta u + (1 - \eta) v \bigr).
\end{align}
\end{subequations}

Now let us describe the Bregman analogue of \tilt, which is based on
inserting an auxiliary correction in the dual mean parameter. Consider
a pair of functions $(\ffun, \bfun)$ that satisfy the pointwise
inclusion $\mu_f(x) + \bfun(x) \in \Omega^*$ pointwise, we define:
\mygraybox{
  \label{DefForwardBregmanIIW}
  {\bf{Bregman-\tilt:}}
\begin{align*}
  \Riskpsi(f,b) \defn & \Exs_{X \sim \Prob} \Bigl[
    D_{\psi^*}\bigl(\mu_f(X) + \bfun(X), \mu^\star(X)\bigr) \Bigr] +
  \lambda \Exs_{\Xtil \sim \Qprob} \Bigl[ D_{\psi^*}\bigl(\mu_f(\Xtil)
    + \bfun(\Xtil), \mu_f(\Xtil)\bigr) \Bigr].
\end{align*}
}
Observe that the first term fits a corrected source prediction to the
oracle conditional mean, whereas the second term penalizes the
correction on target inputs.  It is a natural generalization of the
\tilt least-squares objective: in particular, if we use the Bregman
function $\psi(z) = \frac{1}{2} z^2$, then we have the equivalences
\begin{align*}
  D_{\psi^*}\bigl(\mu_f(X) + \bfun(X), \mu^\star(X)\bigr) \Bigr] \; =
    \; \big(\mu_f(X) + \bfun(X) - \mu^\star(X) \big)^2 \quad
    \mbox{and} \quad D_{\psi^*}\bigl(\mu_f(\Xtil) + \bfun(\Xtil),
    \mu_f(\Xtil)\bigr) = \bfun^2(\Xtil),
\end{align*}
so that we recover the original \tilt least-squares formulation.

To state the decomposition, introduce the natural parameters
\begin{subequations}
\begin{align}
  \thetastar(x) \coloneqq \nabla\psi^*\bigl(\mu^\star(x)\bigr),
  \qquad \theta_f(x) \coloneqq \nabla\psi^*\bigl(\mu_f(x)\bigr)=f(x),
\end{align}
along with their $\lambda$-dependent barycenter
\begin{align}
\bar{\theta}_{f,\lambda}(x) \coloneqq \Vlam(x)\thetastar(x) +
\bigl(1-\Vlam(x)\bigr)\theta_f(x).
\end{align}
\end{subequations}

\mygraybox{
\begin{proposition}[Exact Bregman decomposition]
\label{PropForwardBregmanIIW}
For any admissible pair $(f,b)$, we have the decomposition
\begin{align}
\label{EqForwardBregmanDecomp}
\Riskpsi(f,b) &= \int J_{\psi}^{\Vlam(x)} \bigl(\thetastar(x),
\theta_f(x)\bigr)d\rho_\lambda(x) + \int D_{\psi} \Bigl(
\bar{\theta}_{f,\lambda}(x), \nabla\psi^*\bigl(\mu_f(x)+b(x)\bigr)
\Bigr)d\rho_\lambda(x).
\end{align}
\end{proposition}
}

\begin{proof}
From classical results, Bregman divergences satisfy the primal-dual
identity
\begin{align*}
D_{\psi^*}(u,v) = D_{\psi}\bigl(\nabla\psi^*(v),
\nabla\psi^*(u)\bigr),
\end{align*}
Using this fact, we can rewrite the integrand that defines
$\Riskpsi(f,b)$ as
\begin{align*}
\pdens(x) D_{\psi} \Bigl( \thetastar(x),
\nabla\psi^*\bigl(\mu_f(x)+b(x)\bigr) \Bigr) + \lambda\qdens(x)
D_{\psi} \Bigl( \theta_f(x), \nabla\psi^*\bigl(\mu_f(x)+b(x)\bigr)
\Bigr).
\end{align*}
For any convex $\phi$ and any $\eta\in[0,1]$, the weighted Jensen
divergence $J_\phi^\eta$ satisfies the identity
\begin{align*}
  \eta D_\phi(u,z)+(1-\eta)D_\phi(v,z) = J_\phi^\eta(u,v) +
  D_\phi\bigl(\eta u+(1-\eta)v,z\bigr).
\end{align*}
We apply this fact with the choices $\phi=\psi$, $\eta=\Vlam(x)$,
$u=\thetastar(x)$, $v=\theta_f(x)$, and
$z=\nabla\psi^*(\mu_f(x)+b(x))$, thereby obtaining the
claim~\eqref{EqForwardBregmanDecomp}.
\end{proof}

\paragraph{Profiled target limit:} We now  describe the natural
generalization of~\Cref{ThmRiskDecomp}.  By inspection of
equation~\eqref{EqForwardBregmanDecomp}, if we minimize the function
$\Riskpsi(f, \bfun)$ over the choice of $\bfun$, then the second term
vanishes and we obtain
\begin{align*}
\Jpsilam(f) \coloneqq \inf_b \Riskpsi(f,b) = \int J_{\psi}^{\Vlam(x)}
\bigl(\thetastar(x), \theta_f(x)\bigr)d\rho_\lambda(x).
\end{align*}
Note that the function $\bflam$ that achieves the infimum is defined
pointwise by the relation $\nabla\psi^*\bigl(\mu_f(x)+\bflam(x)\bigr)
= \bar{\theta}_{f,\lambda}(x)$, or equivalently, via the Legendre
properties 
\begin{align}
\bflam(x) = \nabla\psi \bigl( \bar{\theta}_{f,\lambda}(x) \bigr) -
\mu_f(x),
\end{align}
where we have used the fact that $\nabla \psi$ is the inverse of
$\nabla \psi^*$.

In the least squares case, taking the limit of the profiled risk
as $\lambda \rightarrow 0^+$ led to the $L^2(\Qprob)$-target risk.
We now state and prove the generalization of this property to
the Bregman setting:
\mygraybox{
\begin{proposition}[Small-$\lambda$ limit]
\label{PropBregmanLimit}
Assume that $\mu^\star(x)$ lies in the interior of $\Omega^*$ for
almost every $x$, and that the displayed quantities are integrable.
Then for any predictor $f$, we have
\begin{align*}
 \lim_{\lambda \rightarrow 0^+} \frac{1}{\lambda}\Jpsilam(f) & =
 \Exs_{X \sim \TargetX} \Bigl[ D_{\psi^*}\bigl(\mu^\star(X),\mu_f(X)
   \bigr) \Bigr].
\end{align*}
\end{proposition}
}

\begin{proof}
For fixed $u,v$ and any $\varepsilon \in [0,1]$, by a first-order
Taylor series expansion, we can write
\begin{align*}
J_\psi^{1-\varepsilon}(u,v) = \varepsilon D_\psi(v,u) +
o(\varepsilon),
\end{align*}
valid as $\varepsilon \rightarrow 0^+$.  We apply this relation
pointwise in $x$ with the choices
\begin{align*}
  \varepsilon_\lambda(x) \defn 1 -
  \Vlam(x) = \frac{\lambda\qdens(x)}{\pdens(x)+\lambda\qdens(x)},
\end{align*}
and $u \defn \thetastar(x)$ and $v \defn \theta_f(x)$.  Doing so
yields
\begin{align*}
    \bigl(\pdens(x) + \lambda \qdens(x)\bigr)
    J_\psi^{\Vlam(x)}(\thetastar(x),\theta_f(x)) = \lambda \qdens
    D_\psi(\theta_f(x),\thetastar(x)) + o(\lambda),
\end{align*}
where have used the equivalence $\bigl(\pdens(x) + \lambda
\qdens(x)\bigr) \varepsilon_\lambda(x) = \lambda \qdens(x)$.  Finally,
we use the duality relation $D_\psi(\theta_f,\thetastar)
=D_{\psi^*}(\mu^\star,\mu_f)$.  Integrating over $x$ completes the
proof.
\end{proof}

\subsection{Special cases}

Let us describe how this general formulation recovers both
least-squares \tilt described in the main body, as well as
a version of $K$-ary classification.

\subsubsection{Least-squares \tilt}

Consider the Legendre function $\psi(z) = \frac{1}{2} \|z\|_2^2$.  It
satisfies the dual relation $\psi = \psi^*$ and $\nabla \psi(z) = z$,
so that $\mu_f(x) = f(x)$ and $\mu^\star(x) = \fstar(x)$.
Consequently, the $\phi$-risk for this choice takes the form
\begin{align*}
\Riskpsi(f, \bfun) = \Exs_{X \sim \SourceX} \Bigl[
  \frac{1}{2}\|f(X)-f^\star(X) + \bfun(X)\|_2^2 \Bigr] + \lambda
\Exs_{\Xtil \sim \TargetX} \Bigl[ \frac{1}{2} \|\bfun(\Xtil)\|_2^2
  \Bigr].
\end{align*}
Up to the irrelevant factor $1/2$, this is exactly the original \tilt
population objective, as described in~\Cref{SecMethods}.

\subsubsection{KL surrogates for logistic regression}

For $K$-class logistic regression, the natural generator is
\begin{align*}
\psi(z) = \log\Bigl(\sum_{k=1}^K e^{z_k}\Bigr), \qquad
\mu_f(x)=\nabla\psi(f(x))=\pi(f(x)),
\end{align*}
where $\pi$ denotes the softmax map. Its conjugate is negative entropy
on the probability simplex,
\begin{align*}
    \psi^*(u)=\sum_{k=1}^K u_k\log u_k,
\end{align*}
and hence $D_{\psi^*}(u,v)=\mathrm{KL}(u\|v)$ corresponds to the
Kullback--Leibler divergence.  Writing $\rho(x) \coloneqq
\Prob(Y=\cdot\mid X=x)$, the population Bregman-\tilt objective
formally becomes
\begin{align}
    \Riskpsi(f,b)
    =
    \Exs_{X \sim \SourceX}
    \Bigl[
        \mathrm{KL}\bigl(\mu_f(X)+b(X)\,\|\,\rho(X)\bigr)
    \Bigr]
    +
    \lambda
    \Exs_{\Xtil \sim \TargetX}
    \Bigl[
        \mathrm{KL}\bigl(\mu_f(\Xtil)+b(\Xtil)\,\|\,\mu_f(\Xtil)\bigr)
    \Bigr].
    \label{EqPopulationKLTILT}
\end{align}
The corresponding profiled objective can be written explicitly. Let
$\theta^\rho(x)\coloneqq\nabla\psi^*(\rho(x))$ be any canonical logit
representative of the oracle class-probability vector. By
\Cref{PropForwardBregmanIIW},
\begin{align*}
    \Jpsilam^{\mathrm{KL}}(f)
    \coloneqq
    \inf_b \Riskpsi(f,b)
    =
    \int
    J_{\psi}^{\Vlam(x)}
    \bigl(\theta^\rho(x), f(x)\bigr)d\rho_\lambda(x).
\end{align*}
For the log-partition generator, this weighted Jensen divergence is
\begin{align*}
    J_{\psi}^{\eta}(\theta,\theta')
    =
    \eta \log\sum_{k=1}^K e^{\theta_k}
    +
    (1-\eta)\log\sum_{k=1}^K e^{\theta'_k}
    -
    \log\sum_{k=1}^K
    e^{\eta\theta_k+(1-\eta)\theta'_k}.
\end{align*}
Choosing the equivalent normalized representatives
$\theta^\rho(x)=\log\rho(x)$ and
$\theta_f(x)=\log\mu_f(x)$, for which both log-partition terms vanish,
gives
the simpler probability-coordinate form
\begin{align*}
    \Jpsilam^{\mathrm{KL}}(f)
    =
    \int
    -\log
    \Biggl\{
        \sum_{k=1}^K
        \rho_k(x)^{\Vlam(x)}
        \mu_{f,k}(x)^{1-\Vlam(x)}
    \Biggr\}
    d\rho_\lambda(x).
\end{align*}
Thus the KL profiled objective is a covariate-weighted Chernoff-type
discrepancy between the oracle class probabilities and the predictor.
In the small-$\lambda$ limit, \Cref{PropBregmanLimit} gives
\begin{align*}
    \frac{1}{\lambda}\Jpsilam^{\mathrm{KL}}(f)
    \longrightarrow
    \Exs_{\Xtil\sim\TargetX}
    \Bigl[
        \mathrm{KL}\bigl(\rho(\Xtil)\,\|\,\mu_f(\Xtil)\bigr)
    \Bigr].
\end{align*}
This expression gives the correct geometry, but it is not yet a
usable training objective. First, hard labels cannot be substituted
directly for $\rho(x)$: the term
$\mathrm{KL}(\mu_f(x)+b(x)\|y)$ is infinite unless the corrected
prediction puts all mass on the observed class. Second, a mean-space
correction must obey the simplex constraint
$\mu_f(x)+b(x)\in\Delta^{K-1}$, which is awkward for neural-network
parametrizations.

\paragraph{The \kltilt surrogate.}
We address these issues by using a soft teacher and by applying the
auxiliary correction in logit space. Let
$\pi_T(z)\coloneqq\softmax(z/T)$ be the temperature-$T$ softmax, and
let $\tau:\Xspace\to\mathbb{R}^K$ be teacher logits. The direct
surrogate for \eqref{EqPopulationKLTILT} is
\begin{align}
    \LossHat_{\kltilt}(f,b;\tau)
    \coloneqq\;&
    \frac{T^2}{\numsource}
    \Sum{i}{\numsource}
    \mathrm{KL}
    \Bigl(
        \pi_T\bigl(f(x_i)+b(x_i)\bigr)
        \,\Big\|\,
        \pi_T\bigl(\tau(x_i)\bigr)
    \Bigr)
    \notag\\
    &+
    \frac{\lambda T^2}{\numtarget}
    \Sum{j}{\numtarget}
    \mathrm{KL}
    \Bigl(
        \pi_T\bigl(f(\xtil_j)+b(\xtil_j)\bigr)
        \,\Big\|\,
        \pi_T\bigl(f(\xtil_j)\bigr)
    \Bigr).
    \label{EqKLTILT}
\end{align}
The first KL term is the teacher-based replacement for the source
Bregman loss, and the second term is the target-side \tilt penalty.

\paragraph{Relation to \kdtilt.}
The \kdtilt objective used in the experiments keeps the same
target-side penalty as \eqref{EqKLTILT}, but replaces the source term
by the usual distillation direction,
\begin{align*}
    \mathrm{KL}
    \Bigl(
        \pi_T\bigl(\tau(x_i)\bigr)
        \,\Big\|\,
        \pi_T\bigl(f(x_i)+b(x_i)\bigr)
    \Bigr),
\end{align*}
optionally mixed with the supervised cross-entropy loss, as in
\eqref{EqForwardStabilizedSurrogate}. Thus \kltilt is the most direct
KL analogue of the Bregman-\tilt construction, while \kdtilt is its
standard knowledge-distillation surrogate.

\section{Additional results from numerical studies}
\label{AppAdditionalResults}

This appendix collects numerical results that complement the main
experiments.  The synthetic plots give diagnostics for aspects that are
harder to read from the main figures: how the one-dimensional solution
changes with $\lambda$, how the learned auxiliary component behaves, and
how the relative sizes of $\Fclass$ and $\Bclass$ affect the fit.  The
CIFAR-100 plots report additional corruption types and the top-5 metric.

\subsection{Synthetic Regression Diagnostics}
\label{AppSyntheticAdditional}

\myparagraph{Effect of $\lambda$ in a misspecified one-dimensional problem}
\textbf{Setup.}  We use the same one-dimensional misspecified linear
problem as in \Cref{FigSyntheticTiltLinearAB}A--B.  The target
distribution is fixed at the Beta$(2,5)$ endpoint, while the source
distribution is moved from the matched case toward the Beta$(5,2)$
endpoint.  The deployed class $\Fclass$ is a degree-three shifted
Legendre linear model, and the auxiliary class $\Bclass$ is a richer
Gaussian-kernel ridge expansion.  For source corruption levels
$0,0.3,0.7$, and $1.0$, we sweep $\lambda$ and report target-test MSE
for \tilt and exact \rl, with source ERM shown as a horizontal
reference.

\begin{figure}[h!]
  \begin{center}
  \widgraph{0.88\textwidth}{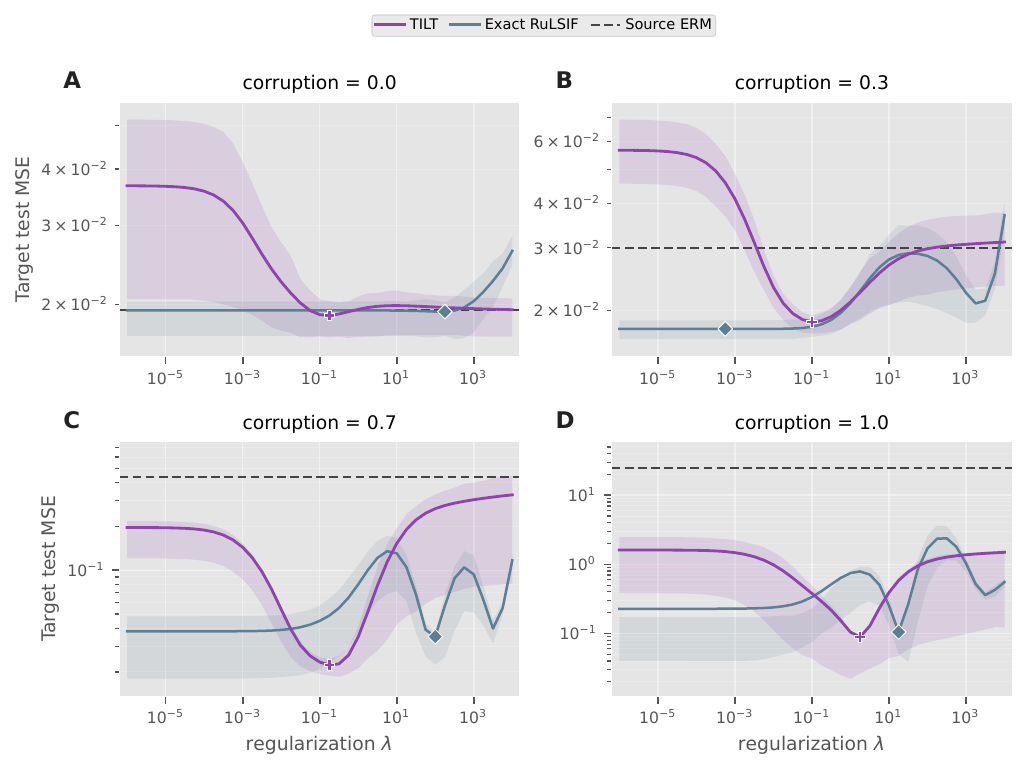}
  \caption{\textbf{Regularization sensitivity in the one-dimensional
      synthetic regression experiment.}  Each panel fixes a source
    corruption level and plots target-test MSE as a function of
    $\lambda$.  Curves show means over trials and shaded bands show
    interquartile ranges.  The \tilt and exact \rl curves have visibly
    different $\lambda$ dependence, and the favorable range for \tilt
    changes with the source corruption level.}
  \label{FigSyntheticLambdaSensitivity}
  \end{center}
\end{figure}

In this one-dimensional problem, the useful range of $\lambda$ shifts
with the corruption level, and the curve for \tilt behaves differently
from exact \rl.  Large $\lambda$ moves the estimator toward the
source-ERM reference, while small $\lambda$ gives the auxiliary
component a weaker target penalty.  The plot is therefore mainly a
diagnostic of how the fitted solution changes with $\lambda$ in this
particular misspecified setting.

\myparagraph{Auxiliary behavior in the one-dimensional problem}
\textbf{Setup.}  This diagnostic inspects the learned auxiliary
component $b$ in the same one-dimensional experiment.  The diagnostics
are computed at corruption levels $0$ and $1$ while sweeping $\lambda$,
and they track
how much energy $b$ carries on source and target covariates together
with its interaction with the source residual structure.

\begin{figure}[h!]
  \begin{center}
  \widgraph{0.95\textwidth}{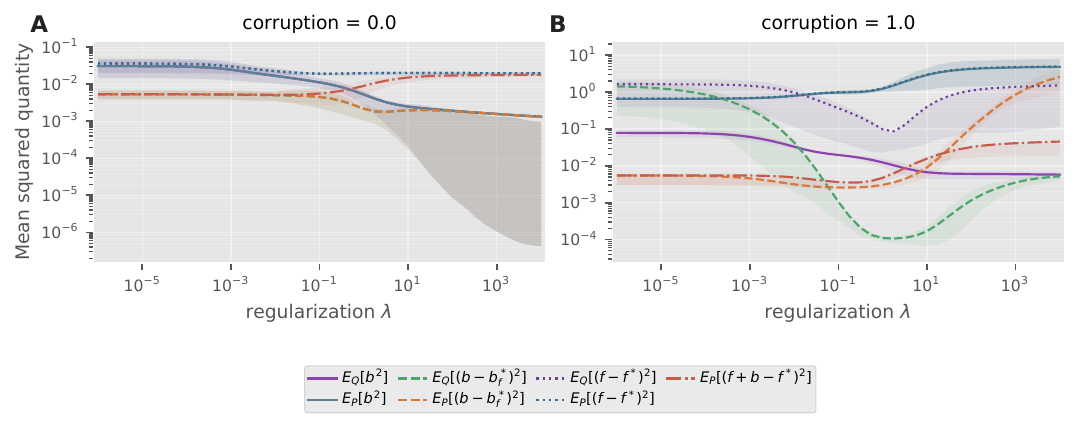}
  \caption{\textbf{Auxiliary-component diagnostics for the synthetic
      regression experiment.}  For the same one-dimensional synthetic
    problem as in \Cref{FigSyntheticTiltLinearAB}, \textbf{A} shows
    diagnostic quantities for the learned auxiliary component $b$ at
    corruption level $0$.  In the matched-domain case, all quantities
    involving $b$ remain small, and increasing $\lambda$ keeps the
    target-side contribution controlled.  \textbf{B} repeats the
    diagnostics at corruption level $1$.  Under the large shift, the
    source- and target-side quantities separate strongly, consistent
    with $b$ fitting source-specific residual structure while the target
    penalty limits its effect on the deployed component $f$.}
  \label{FigSyntheticTiltLinearDiagnostics}
  \end{center}
\end{figure}

\Cref{FigSyntheticTiltLinearDiagnostics} complements the preceding
$\lambda$ sweep.  When the domains are nearly matched, the auxiliary
component remains small.  Under the larger shift, the diagnostics are
consistent with $b$ carrying more source-specific residual variation,
while the target-side penalty limits how much of that variation is
inherited by the deployed predictor $f$.

\myparagraph{Changing the auxiliary dimension in a bounded-ratio problem}
\textbf{Setup.}  To isolate the role of $\lambda$ and the complexity of
$b$, we run a finite-linear diagnostic under a bounded density ratio.
The source distribution is uniform on $[0,1]$, and the target
distribution is obtained by a smooth raised-cosine density ratio with
$0.13 \leq q(x)/p(x) \leq 7.8$.  We use $320$ source samples, $320$
target covariates, Gaussian noise with standard deviation $0.08$, and
average over $100$ trials.  The deployed class $\Fclass$ is a shifted
Legendre basis of dimension $d_f$, while the auxiliary class
$\Bclass$ is a real Fourier basis of dimension $d_b$; the figure
compares $(d_f,d_b)\in\{(20,8),(8,20),(8,8),(20,20)\}$.

\begin{figure}[h!]
  \begin{center}
  \widgraph{0.98\textwidth}{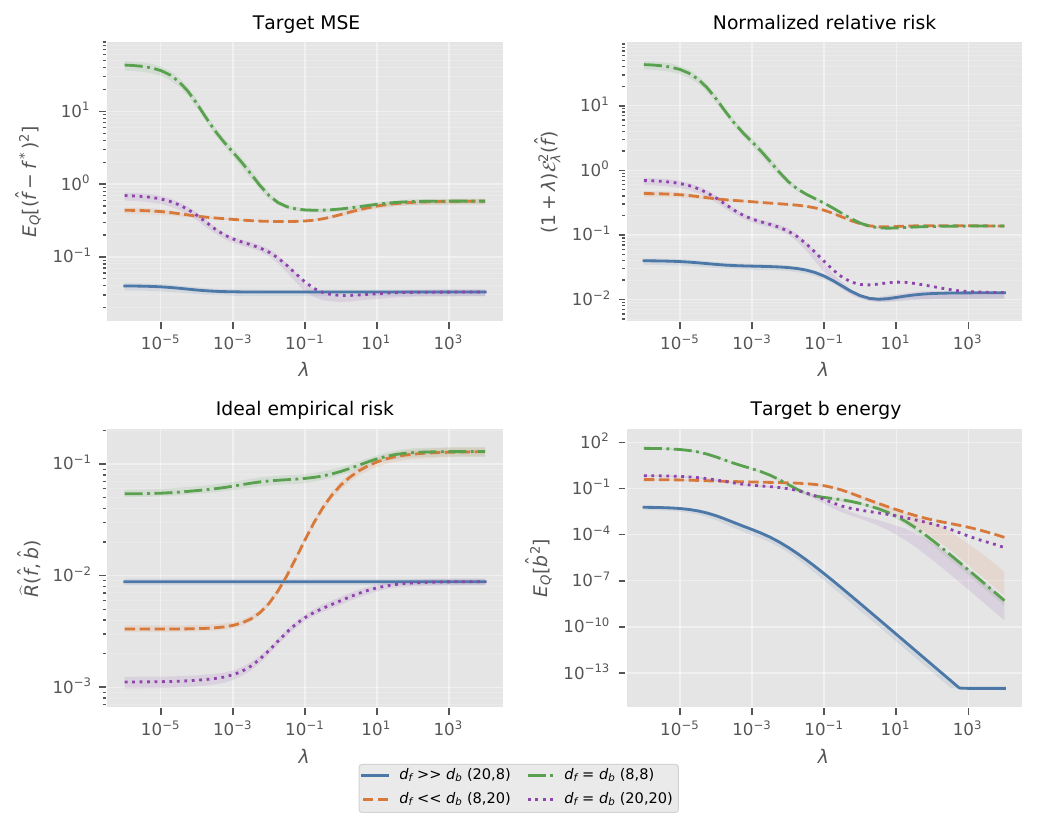}
  \caption{\textbf{Finite-linear $\lambda$ sweep under bounded density
      ratio.}  The four panels reproduce the first four diagnostics from the
    dimension sweep, except that the second panel reports
    $(1+\lambda)\Errlam$ rather than the raw $\Errlam$.  When $d_f=20$, the
    task is effectively well specified for the deployed class.  In this case,
    a smaller auxiliary class keeps the small-$\lambda$ solution closer to
    source ERM, whereas a rich auxiliary class can be worse at small
    $\lambda$.  When $d_f=8$, the deployed class is misspecified and a richer
    auxiliary class improves performance in this sweep.}
  \label{FigBoundedRatioLinearLambdaSweep}
  \end{center}
\end{figure}

\Cref{FigBoundedRatioLinearLambdaSweep} keeps the density ratio bounded
and varies the dimensions of $f$ and $b$.  When $d_f=20$, the deployed
class is close to well specified; in this case a smaller auxiliary class
keeps the small-$\lambda$ solution close to source ERM, whereas making
both $f$ and $b$ rich can be worse at very small $\lambda$.  When
$d_f=8$, the deployed class is misspecified and the richer auxiliary
class is helpful.  These finite-dimensional curves are consistent with
the view that $\lambda$ controls how much auxiliary flexibility is used.

\myparagraph{A lower-dimensional neural control}
\textbf{Setup.}  This lower-dimensional neural-network control uses
covariates in $[0,1]^{16}$ with independent
product-Beta source and target marginals: the target endpoint is
Beta$(2,5)^{16}$ and the source endpoint is Beta$(5,2)^{16}$.  The
regression function is generated by a fixed ReLU teacher network, and
we use $4096$ labeled source samples, $4096$ unlabeled target
covariates, target-test sets of size $65536$, and $20$ trials.  The
deployed predictor $f$ is a one-hidden-layer ReLU network of width
$16$; \tilt may additionally use a deeper and wider auxiliary network
$b$.  Source ERM, exact \iw, exact \rl, kernel-estimated \rl, and
\tilt are tuned over the same validation protocol used for the other
synthetic experiments.

\begin{figure}[h!]
  \begin{center}
  \widgraph{0.44\textwidth}{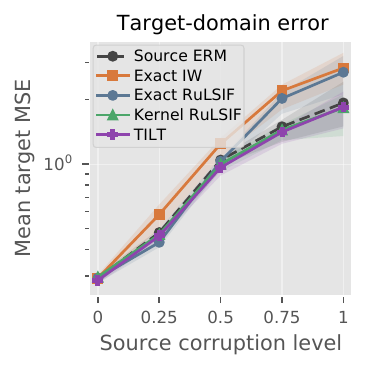}
  \caption{\textbf{Well-specified beta-product neural control.}
    Target MSE is reported for a $16$-dimensional beta-product
    covariate-shift problem with a fixed ReLU-teacher regression
    function.  Unlike the high-dimensional weak-class experiment in
    \Cref{FigSyntheticTiltLinearAB}, the deployed neural class is
    sufficiently rich for this lower-dimensional task.  Source ERM and
    kernel-estimated \rl therefore perform similarly to \tilt in this
    setting.}
  \label{FigBetaProductWellSpecified}
  \end{center}
\end{figure}

\Cref{FigBetaProductWellSpecified} is included as a control.  In this
lower-dimensional setting, the deployed neural class is already rich
enough that source ERM and kernel-estimated \rl remain competitive with
\tilt over much of the shift range.  Thus this particular
well-specified regime is not a stress test for \tilt.

\subsection{Additional CIFAR-100 Results}
\label{AppCifar100Additional}

\Cref{FigCifar100SingleCorruptionLoss} shows that the single-corruption
results are not uniform across corruption types.  For Gaussian and
defocus blur, \kltilt overtakes vanilla \kd once the shift is large
enough.  Contrast is milder in this sweep: the teacher remains accurate
over most of the path, and the benefit of \kltilt appears only around
severity $1.5$.  Pixelate shows little separation between \kd and
\kltilt in this configuration.  Thus these plots support the main
CIFAR-100 trend for several corruptions, but they also make clear that
the size of the gain depends on the corruption.

\myparagraph{Metric robustness on the main CIFAR-100 path}
\textbf{Setup.}  We report the top-5 metric for the same
mixed-corruption CIFAR-100 experiment used in \Cref{FigCifar100}.  The
source is clean CIFAR-100; the target path jointly changes brightness,
contrast, color, and Gaussian blur; the methods are source ERM, \kd,
\kdtilt, and \kltilt with the same validation-selected
hyperparameters and seed aggregation as in the main comparison.
\begin{figure}[h!]
  \begin{center}
  \widgraph{0.52\textwidth}{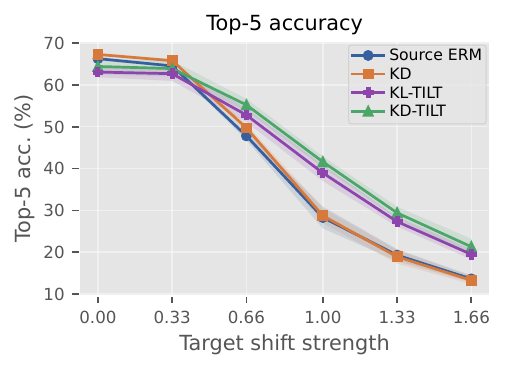}
  \caption{\textbf{Top-5 accuracy on the CIFAR-100 covariate-shift
      experiment.}  This figure complements \Cref{FigCifar100} by
    reporting top-5 accuracy on the shifted target-test set for the
    same source ERM, \kd, \kdtilt, and \kltilt methods.  Solid curves
    show means over seeds and shaded bands indicate $\pm1$ standard
    deviation.}
  \label{FigCifar100Top5}
  \end{center}
\end{figure}
\Cref{FigCifar100Top5} gives the corresponding top-5 view of the main
CIFAR-100 path.  The ordering is similar to the top-1 and loss plots:
tilted distillation is competitive at small shifts and becomes better
than source ERM and vanilla \kd once the target corruption is larger.

\myparagraph{CIFAR-100 behavior across individual corruption types}
\textbf{Setup.}  This experiment tests whether the CIFAR-100 behavior in
\Cref{FigCifar100} depends on the particular mixed corruption path used
in the main text.  The source domain is clean CIFAR-100, with
the same source-trained ResNet-20 teacher, weak CNN student $f$, and
ResNet-20 auxiliary logit model $b$.  On the target side, we apply one
corruption type at a time: Gaussian blur, defocus blur, contrast, or
pixelate, with severity in $\{0,0.5,1.0,1.5,2.0\}$.  Brightness is
omitted to keep the panel readable.  For \kltilt, we report
the best validation-selected value over $\lambda\in\{10,30,100\}$, and
all student curves average three seeds.

\begin{figure}[h!]
  \begin{center}
  \widgraph{0.67\textwidth}{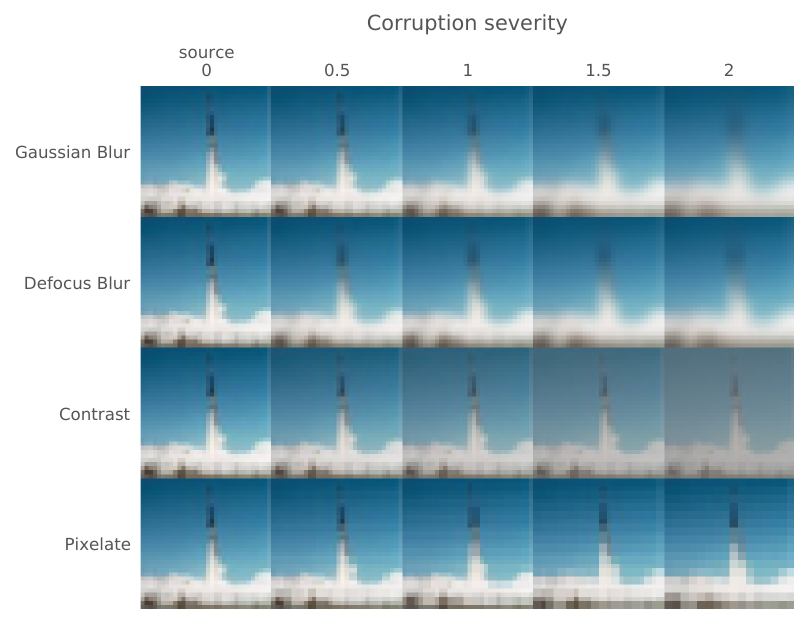}
  \par\vspace{0.35em}
  \widgraph{0.60\textwidth}{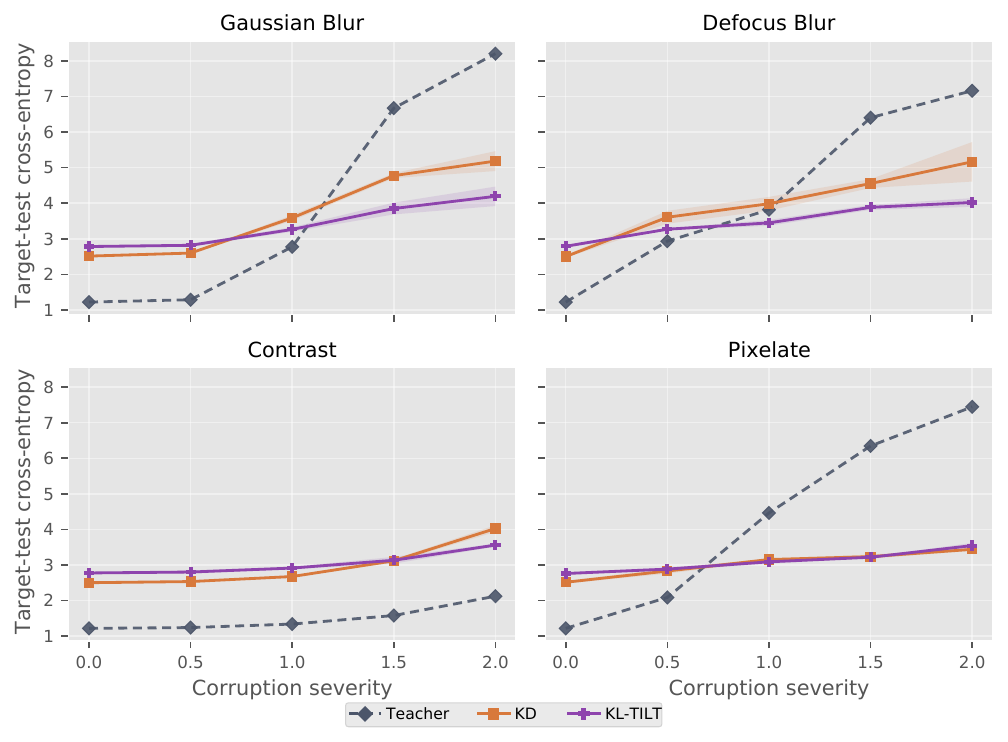}
  \caption{\textbf{Target-test cross-entropy under single-corruption
      CIFAR-100 shifts.}  This figure repeats the CIFAR-100
    distillation experiment while varying the target corruption type
    one at a time.  The top grid visualizes the same CIFAR-100 test
    image under Gaussian blur, defocus blur, contrast, and pixelate
    corruptions; columns correspond
    to the severity values used in the sweep.  The bottom panels report
    target-test cross-entropy.  The teacher curve is the source-trained
    ResNet-20 evaluated directly on the shifted target-test set.  The
    \kd and \kltilt curves show means over three seeds with $\pm1$
    standard-deviation bands; \kltilt denotes the validation-selected
    Bregman-IIW/KL-TILT model.}
  \label{FigCifar100SingleCorruptionLoss}
  \end{center}
\end{figure}

\end{document}